\documentclass[lettersize,journal]{IEEEtran}

\usepackage{subfiles}

\usepackage[style=ieee,natbib=true,backend=bibtex,doi=false,isbn=false,url=false,eprint=false,mincitenames=1,maxcitenames=1,uniquelist=false,alldates=year,maxbibnames=10]{biblatex}
\bibliography{references}
\AtEveryBibitem{%
  \clearfield{note}%
  \clearlist{language}
}

\usepackage{float}
\usepackage{subfig}
\usepackage[utf8]{inputenc}
\usepackage[T1]{fontenc}    
\usepackage[hidelinks]{hyperref}       
\usepackage{url}            
\usepackage{booktabs}       
\usepackage{amsmath, amsfonts, amsthm, amssymb} 
\usepackage{nicefrac}
\usepackage{microtype} 
\usepackage{lipsum}
\usepackage{graphicx}
\usepackage{nicefrac}

\usepackage{enumerate}
\usepackage{wrapfig}
\usepackage{mathrsfs}
\usepackage{xcolor}         
\definecolor{MyBlue}{HTML}{9467bd}
\definecolor{MyRed}{HTML}{d62728}
\definecolor{MyGreen}{HTML}{01825d}

\newtheorem{theorem}{Theorem}

\newtheorem{proposition}[theorem]{Proposition}
\newtheorem{lemma}[theorem]{Lemma}

\newtheorem{conjecture}[theorem]{Conjecture}

\newcommand{\N}{\mathcal{N}}

\newcommand{\R}{\mathbb{R}}
\newcommand{\E}[2]{\mathbb{E}_{{#1}}\left[{#2}\right]}
\newcommand{\Exp}[1]{\mathbb{E}\left[{#1}\right]}

\newcommand{\Cov}[1]{\text{Cov}\left[#1\right]}
\newcommand{\Var}[1]{\mathbb{V}\left[{#1}\right]}
\newcommand{\trnsp}{\mathsf{T}}
\newcommand{\mhat}[1]{\widehat{#1}}
\newcommand{\tr}{\text{tr}}

\newcommand{\brackets}[1]{\left\{#1\right\}}

\newcommand{\norm}[1]{\|#1\|}
\newcommand{\pinv}{\dagger}
\newcommand{\bigO}{\mathcal{O}}

\makeatletter
\newcommand{\longdash}[1][2em]{%
  \makebox[#1]{$\m@th\smash-\mkern-7mu\cleaders\hbox{$\mkern-2mu\smash-\mkern-2mu$}\hfill\mkern-7mu\smash-$}}
\makeatother
\newcommand{\omitskip}{\kern-\arraycolsep}

\title{Overparameterized Linear Regression\\ under Adversarial Attacks}
\author{Antônio H. Ribeiro  and Thomas B. Sch\"on \thanks{The authors are with the Department of Information Technology, Uppsala University, Sweden (emails: \href{mailto:antonio.horta.ribeiro@it.uu.se}{antonio.horta.ribeiro@it.uu.se}, \href{mailto:thomas.schon@it.uu.se}{thomas.schon@it.uu.se})}}

\begin{document}

\maketitle

 \vspace{-4pt}
\begin{abstract}
We study the error of linear regression in the face of adversarial attacks. In this framework, an adversary changes the input to the regression model in order to maximize the prediction error. We provide bounds on the prediction error in the presence of an adversary as a function of the parameter norm and the error in the absence of such an adversary. We show how these bounds make it possible to study the adversarial error using analysis from non-adversarial setups. The obtained results shed light on the robustness of overparameterized linear models to adversarial attacks. Adding features might be either a source of additional robustness or brittleness. On the one hand, we use asymptotic results to illustrate how double-descent curves can be obtained for the adversarial error. On the other hand, we derive conditions under which the adversarial error can grow to infinity as more features are added, while at the same time, the test error goes to zero. We show this behavior is caused by the fact that the norm of the parameter vector grows with the number of features. It is also established that $\ell_\infty$ and $\ell_2$-adversarial attacks might behave fundamentally differently due to how the $\ell_1$ and $\ell_2$-norms of random projections concentrate. We also show how our reformulation allows for solving adversarial training as a convex optimization problem. This fact is then exploited to establish similarities between adversarial training and parameter-shrinking methods and to study how the training might affect the robustness of the estimated models.
\end{abstract}
\vspace{-3pt}

\section{Introduction}
\label{sec:intro}

As machine learning models start to be considered for critical applications such as medical settings~\citep{rajpurkar_ai_2022} or autonomous driving~\citep{hussain_autonomous_2019}, their vulnerabilities and brittleness become a pressing concern~\citep{hendrycks_unsolved_2022}. The adversarial attack framework is popular for studying these issues. It considers inputs contaminated with small disturbances deliberately chosen to maximize the model error. The susceptibility of state-of-the-art neural network models to very small input modifications~\citep{bruna_intriguing_2014} gave the framework a lot of attention from the research community. 

There is a  conflicting view on the relationship between high-dimensionality and model robustness to adversarial attacks that served as the driving force for this work. On the one hand, high-dimensionality is pointed out as a source of vulnerability to adversarial attacks~\cite{gilmer_adversarial_2018,tsipras_robustness_2019,ilyas_adversarial_2019}. On another hand, a new line of work has established the advantages of high-dimensionality: the study of
\textit{double-descent} curves show that it is sometimes possible to obtain improvements in performance if we continue to increase the model size beyond the point of interpolation. The phenomenon is counterintuitive because it shows that under certain conditions overfitting the dataset can also be benign to model generalization and performance; precise conditions for \textit{benign overfitting} are provided by~\citet{bartlett_benign_2020}. The idea also applies to robustness analysis and increasing the model size can be a recipe to obtain more robust models---as shown by~\citet{bubeck_universal_2021} using isoperimetric inequalities.

Linear models are a natural setting to study the role of high-dimensions in the robustness to adversarial attacks. Not only can linear models be made vulnerable to adversarial attack~\citep{goodfellow_explaining_2015}, but the double-descent and the benign overfitting phenomenon can also be observed in a purely linear setting~\citep{bartlett_benign_2020,hastie_surprises_2019}. Indeed, there is a growing body of work that study fundamental properties of adversarial attacks in linear models~\cite{taheri_asymptotic_2021,javanmard_precise_2020,hassani_curse_2022,min_curious_2021,yin_rademacher_2019}.

In this paper, we consider adversarial attacks in the context of linear regression. Given an input $x_0\in \R^m$ and an output $y_0 \in \R$, a linear model $\mhat{\beta}$ makes a prediction $\mhat{\beta}^\trnsp x$. The adversary modifies the input with a disturbance $\Delta x$ such that, for a magnitude of at most $\delta$, it maximizes the squared adversarial error,
\begin{equation}
    \label{eq:adversarial_error}
    \max_{\|\Delta x\|_2 \le \delta}(y_0 - (x_0 + \Delta x)^\trnsp\mhat{\beta})^2.
\end{equation}
We refer to the above attack as $\ell_2$-adversarial attack since it constrains the attack to a ball in the $\ell_2$-norm. More generally, we consider the framework of $\ell_p$-adversarial attacks, which contains the commonly used $\ell_2$ and $\ell_\infty$ adversaries as special cases~\citep{bruna_intriguing_2014,goodfellow_explaining_2015,gilmer_adversarial_2018,tsipras_robustness_2019,ilyas_adversarial_2019, kurakin_adversarial_2018, fawzi_analysis_2018, ilyas_adversarial_2019, yuan_adversarial_2019}.

Adversarial examples in linear regression have been studied before, e.g., \citet{javanmard_precise_2020} presents an exact asymptotic analysis. 
Here we use an approximate analysis instead. The quantity of interest is the \emph{adversarial risk}, i.e., the expected value of the squared adversarial error displayed in Eq.~\eqref{eq:adversarial_error}. We approximate the adversarial risk by the sum of risk and the parameter norm $\|\hat\beta\|$ and show that the analysis of these two components is enough to explain high-dimensionality both as a source of brittleness and as a potential recipe for producing robust models. 
This setting can be used to gain insight into the problem. Its simplicity offers extra flexibility for quickly navigating between different setups and reusing results from other papers, as we illustrate throughout our study.

\vspace{-6pt}
\subsection*{Contributions}
\noindent
This paper makes the following contributions:
\begin{enumerate}
    \item We show that \emph{the analysis of adversarial attacks can be simplified in linear regression problems}.
    The adversarial error in Eq.~\eqref{eq:adversarial_error} can be rewritten as:
    \begin{equation*}
        \left(|y_0 - x_0^\trnsp\mhat{\beta}| + \delta\|\mhat{\beta}\|_2\right)^2
    \end{equation*}
    This is proved in Section~\ref{sec:advrisk-linear-regresion}, where we also discuss the implications of the reformulation.
    \item In Section~\ref{sec:adversarial-robustness},  we show that the ratio between $\ell_2$-adversarial risk and   $(\text{risk} + \delta^2\|\mhat{\beta}\|_2^2)$ is always a factor between~1 and~2. We use this approximation to analyse the adversarial risk:
    \begin{enumerate}
        \item  The minimum-norm solution is commonly used to select overparameterized models in connection to the study of the double-descent phenomenon~\citep{belkin_reconciling_2019,bartlett_benign_2020}. We use asymptotic and non-asymptotic analysis to show for this estimate that  $\|\mhat{\beta}\|_2$ decreases as we add more features. This fact is used to obtain cases where we can observe \emph{double-descent in adversarial scenarios}.
        \item We generalize the analysis to other $\ell_p$-norms. While $\|\mhat{\beta}\|_2$ decreases as we add more features, the $\ell_1$-norm, $\|\beta\|_1$, does not. This is used to explain why more features increase the models \emph{vulnerability to some types of adversarial attacks}: i.e., in Fig. 1 the model becomes more vulnerable to $\ell_\infty$-adversarial attacks as more features are added, but not to $\ell_2$-adversarial attacks.
        \end{enumerate}
      \item \emph{Adversarial training} is a standard method to produce models that are robust to adversarial attacks~\citep{madry_deep_2018}. In Section~\ref{sec:adversarial-training}, we show how adversarial training affects the conclusion obtained for minimum-norm solutions.
        \begin{enumerate}
          \item We show how our formulation allows for  solving adversarial training in linear regression as a~\emph{convex optimization} problem.
          \item We compare and establish similarities with \emph{ridge regression} and~\emph{lasso}~\citep{tibshirani_regression_1996}.
          \item We study how adversarial training  and parameter shrinking methods affect the parameter norm grow with the number of features. We use our observations to explain when it effectively prevents vulnerability to adversarial attacks.
        \end{enumerate}
\end{enumerate}

 \begin{figure}[t]
    \centering
    \includegraphics[width=0.5\textwidth]{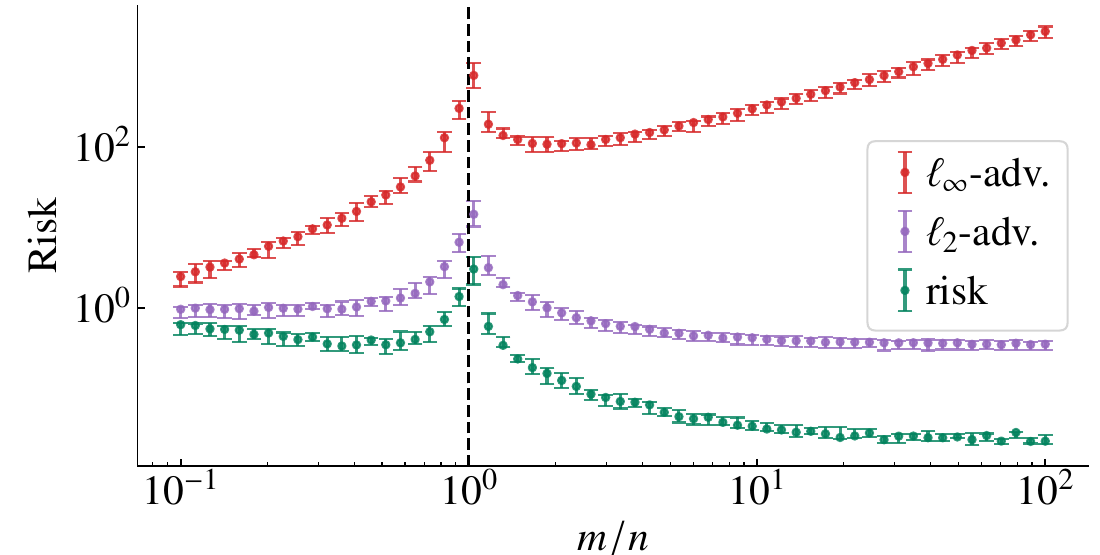}
    \caption{\emph{Double-descent \textit{vs} robustness-accuracy trade-off}. On the $x$-axis we have the ratio between  the number of features $m$ and the number of training datapoints $n$. The model out-of-sample prediction \textit{risk} (when there is no adversarial disturbance) continually decreases in the overparameterized region, achieving significantly better results than in the underparameterized region. On the other hand, increasing the number of features continuously produce worse adversarial robustness for the risk of $\ell_\infty$-adversarial attacks. For an $\ell_2$-adversary, however, the model actually benefits from operating in the overparameterized region and larger models yield better adversarial robustness.  The precise setup for this example is provided in Section~\ref{sec:latent-space-model}.}
    \label{fig:three-types-of-risk}
\end{figure}

\section{Problem formulation}

Consider a training dataset $\mathcal{S} = \{(x_i, y_i)\}_{i=1}^{n}$ consisting of $n$ i.i.d. datapoints of dimension $\R^m \times \R$, sampled from the distribution ${(x_i, y_i)\sim P_{x, y}}$. To this data, we fit a \emph{linear model} from the function class $\left\{\tilde{\beta}^\trnsp x \mid \tilde{\beta} \in \R^m \right\}$. We use $\mhat{\beta}$ to denote the parameter estimated from the training data. The estimation method is detailed in what follows.

We will use subscripts to denote the source of randomness considered in the conditional expectation. For instance, let $x, y, z, w$ be random variables, we use $\E{x, y}{f(x, y, z, w)}$ to denote the expectation with respect to $x, y$ and conditioned on the variables that are not explicitly mentioned in the subscript, here $z$  and $w$. 

Let $(x_0, y_0)\sim P_{x, y}$ be a point not seen during training and independently sampled from the same distribution as the rest of the data. We denote the out-of-sample prediction \emph{risk} by 
\begin{equation}
    R(\mhat{\beta}) = \E{x_0, y_0}{(y_0 - x_0^\trnsp\mhat{\beta})^2}.
  \end{equation}
The $\ell_p$-\emph{adversarial risk} is defined as
\begin{equation}
    \label{eq:adversarial-risk}
    R^{\text{adv}}_p(\mhat{\beta})= \E{x_0, y_0}{\max_{\|\Delta x_0\|_p \le \delta}(y_0 - (x_0 + \Delta x_0)^\trnsp\mhat{\beta})^2},
  \end{equation}
which is the risk when the model is subject to a disturbance $\Delta x_0$ that results in the worst possible performance inside the region $\|\Delta x_0\|_p\le \delta$.
Here, we use $\|\cdot\|_p$ to denote the $\ell_p$-norm of a vector. That is, given a vector $a \in \R^n$, $\|a\|_p = \left(\sum_{i=1}^n |a_i|^p\right)^{\frac{1}{p}}$ and $\|a\|_\infty = \max_{1\le i\le n}{|a_i|}$. The design parameter $\delta$ is the maximum size for the adversarial perturbation.

Moreover, the \emph{empirical risk} is denoted by
\begin{equation}
   \label{eq:empirical-risk}
    \mhat{R}(\mhat{\beta}) = \frac{1}{n}\sum_{i=1}^n{(y_i - x_i^\trnsp\mhat{\beta})}^2,
  \end{equation}
where the expectation w.r.t. the true distribution is replaced by the average over the observed training samples. We use a similar notation for the \emph{empirical adversarial risk},  $\mhat{R}^{\text{adv}}_p(\mhat{\beta})$.

\section{Adversarial risk in linear regression}
\label{sec:advrisk-linear-regresion}

In this paper, we show that the adversarial risk for linear regression can be simplified. The following lemma gives a quadratic form for the adversarial risk defined in Eq.~\eqref{eq:adversarial-risk}.

\begin{lemma}
  \label{thm:advrisk-closeform}
  Let $q$ be a positive real number for which $\frac{1}{p} + \frac{1}{q} = 1$, then\footnote{The result still holds for the pair of values ${(p = 1, q = \infty)}$ and ${(p = \infty, q = 1)}$.}
\begin{equation}
  \label{eq:closeform-adv-risk}
  R_p^{\text{adv}}(\mhat{\beta}) =  \E{x_0, y_0}{ \left(|y_0 - x_0^\trnsp\mhat{\beta}| + \delta\|\mhat{\beta}\|_q\right)^2}.
\end{equation}
\end{lemma}

The lemma above turns out to be a useful tool when analysing robustness to adversarial attacks and adversarial training.
A contribution of this paper is to show how it can be used in various situations.
In Section~\ref{sec:adversarial-robustness}, Lemma~\ref{thm:advrisk-closeform} is used to analyse the adversarial robustness of linear models and the interplay between overparametrization and robustness.  In Section~\ref{sec:adversarial-training}, we show how the formula allows for an efficient method for adversarial training using convex programming.

The proof of the lemma is based on Hölder's inequality ${|\beta^\trnsp\Delta x| \le \|\beta\|_p \|\Delta x\|_q}$ and on the fact that, given $\beta$, we have $\Delta x$ such that the equality holds.
The next proposition gives the precise construction that we will make use of.

\begin{proposition}
  \label{thm:tight-holder}
  Given $p, q \in (1, \infty)$ such that $1/p + 1/q = 1$ and $\beta\in \R^m$, $\Delta x\in \R^m$. We have $|\beta^\trnsp\Delta x| = \|\beta\|_q \|\Delta x\|_p$ when the $i$-th component is $\Delta x_i = \text{sign}(\beta_i) |\beta_i|^{q/p}$.  Moreover, if $\Delta x_i =  \text{sign}(\beta_i)$ for every $i$, then $|\beta^\trnsp\Delta x| = \|\beta\|_1 \|\Delta x\|_\infty$. If
    \begin{equation*}
    \Delta x_i = \frac{s_i}{\sum_i s_i} 
        \quad \text{for}\quad
        s_i=
        \begin{cases}
        1 & \text{if }\beta_i = \max_i \beta_i\\
        0 & \text{otherwise}
        \end{cases}
    \end{equation*}
   then $|\beta^\trnsp\Delta x| = \|\beta\|_\infty \|\Delta x\|_1$.
\end{proposition}

The above proposition is well-known. See, for instance, Exercise 4, Section 2.4 of \citet{ash_probability_2000}. 
Indeed, most adversarial attacks are constructed based on it. For instance, the Fast Gradient Sign Method (FGSM)~\citep{goodfellow_explaining_2015} linearizes the neural network and applies the above construction for $p = \infty$ to obtain the adversarial perturbation. For completeness, we also provide a proof of the proposition in the appendix. We are now ready to prove Lemma~\ref{thm:advrisk-closeform}.

\begin{proof}[Proof of Lemma~\ref{thm:advrisk-closeform}]
For $1\le p\le \infty$, let $q$ be a positive real number such that $\frac{1}{p} + \frac{1}{q} = 1$. Let $e_0 = y_0 - x_0^\trnsp\mhat{\beta}$. After some algebraic manipulation,~\eqref{eq:adversarial-risk} can be rewritten as
\begin{equation*}
    \label{eq:adversarial_risk_expansion}
    R_p^{\text{adv}}(\mhat{\beta}) = \E{x_0, y_0}{e_0^2 + \max_{\|\Delta x_0\|_p \le \delta}\left((\Delta x_0^\trnsp\mhat{\beta})^2 - 2e_0 \Delta x_0^\trnsp \mhat{\beta} \right)}.
\end{equation*}
In turn, if we define $r = \Delta x_0^\trnsp \mhat{\beta} $ and use the fact that $\|\Delta x_0\|_p \le \delta$,  H\"older's inequality yields $$|r|  \le \delta \|\mhat{\beta}\|_q, $$ for $q$ satisfying  $1/p + 1/q = 1$. Since Proposition~\ref{thm:tight-holder} guarantees that we can always choose vectors such that the equality holds, the term inside the expectation is equal to $M + e_0^2$, where \[M = \max_{|r| \le \delta \|\mhat{\beta}\|_q} (r^2 - 2 e_0 r).\]
Now, the maximum is attained at $r = - \delta \|\mhat{\beta}\|_q$ if $e_0 \ge 0$ and at $r = \delta \|\mhat{\beta}\|_q$ if $e_0 < 0$. Hence, $M = \delta^2 \|\mhat{\beta}\|_q^2 +2 \delta \|\mhat{\beta}\|_q |e_0|$ and
\begin{equation}
  R_p^{\text{adv}}(\mhat{\beta}) =  \E{x_0, y_0}{\delta^2\|\mhat{\beta}\|_q^2  + 2\delta\norm{\mhat{\beta}}_q |e_0| +  |e_0|^2}.
\end{equation}
\end{proof}

\section{Adversarial robustness of linear models}
\label{sec:adversarial-robustness}
In this section, we analyse the adversarial risk, by exploiting Lemma~\ref{thm:advrisk-closeform}. Expanding~\eqref{eq:closeform-adv-risk} and using the linearity of the expectation operator we obtain 
\begin{equation}
    \label{eq:closeform-adv-risk1}
  R_p^{\text{adv}}(\mhat{\beta}) = \Bigg(R(\mhat{\beta}) + \delta^2\|\mhat{\beta}\|_q^2\Bigg)  + 2\delta\norm{\mhat{\beta}}_q \E{x_0, y_0}{|y_0 - x_0^\trnsp\mhat{\beta}|} .
\end{equation} 
The term inside the parenthesis is present in most regularized settings and we can naturally see similarities with the cost function of lasso and ridge regression for $q=1$ and $2$, respectively. The last term is always positive and can be upper bounded using Jensen's inequality. Hence, it follows that
\begin{equation}
  \label{eq:ineq_adversarial_risk}
   R(\mhat{\beta}) + \delta^2\|\mhat{\beta}\|_q^2 \le R_p^{\text{adv}}(\mhat{\beta}) \le \left(\sqrt{R(\mhat{\beta})} + \delta  \|\mhat{\beta}\|_q \right)^2.
 \end{equation}
The inequality can be further simplified using the fact that for all $a,b \ge 0$ the inequality $(a+b)^2 \le 2(a^2+b^2)$ holds. Hence,
$$1 \le \frac{R_p^{\text{adv}}(\mhat{\beta})}{ R(\mhat{\beta}) + \delta^2\|\mhat{\beta}\|_q^2} \le 2.$$
That is, the adversarial risk is between 1 and 2 times the quantity $R(\mhat{\beta}) + \delta^2\|\mhat{\beta}\|_q^2$. Such bounds allow for the analysis of adversarial robustness from values that are often obtained from other analyses (the risk and the parameter norm). 
For instance, from the double-descent literature, it is well-known that the $\ell_2$-norm of the estimated parameter often also exhibits a double-descent behavior as a function of the number of features. This is observed experimentally, for instance, by~\citet{belkin_reconciling_2019}. Hence, the above approximation offers an easy way to understand why we can expect double-descent behavior for the $\ell_2$-adversarial risk.

The approximation is also enough to justify the potential brittleness of high-dimensional models. It gives sufficient and necessary conditions for a model with good test performance to be made vulnerable to adversarial attacks as more features are added. For instance, Eq.~\eqref{eq:ineq_adversarial_risk} implies that:
\textit{For a sufficiently small risk ${R(\mhat{\beta}) < \epsilon}$,  the adversarial risk ${R_p^{\text{adv}}(\mhat{\beta}) \rightarrow \infty}$ as ${m \rightarrow \infty}$ if and only if: $\delta \|\mhat{\beta}\|_q \rightarrow \infty$.} The next example demonstrates a fairly simple case where such behavior can be observed.

\subsection{Motivating example: weak features}
\label{sec:unidimensional-latent-space}

 \begin{figure}[t]
    \centering
   \subfloat[Optimal predictor]{ \includegraphics[width=0.45\textwidth]{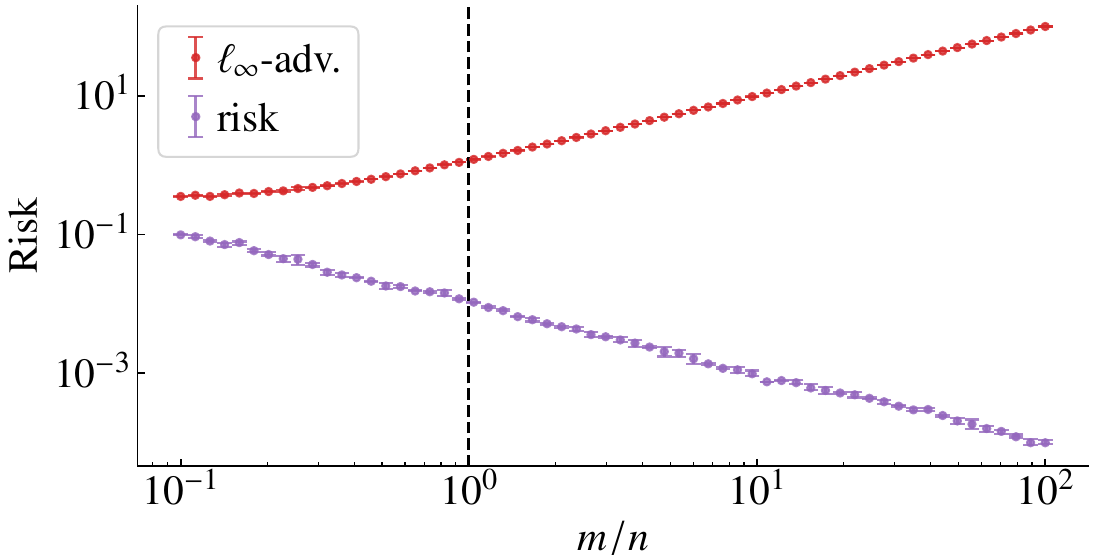}}\\
    \subfloat[Min $\ell_2$-norm predictor]{ \includegraphics[width=0.45\textwidth]{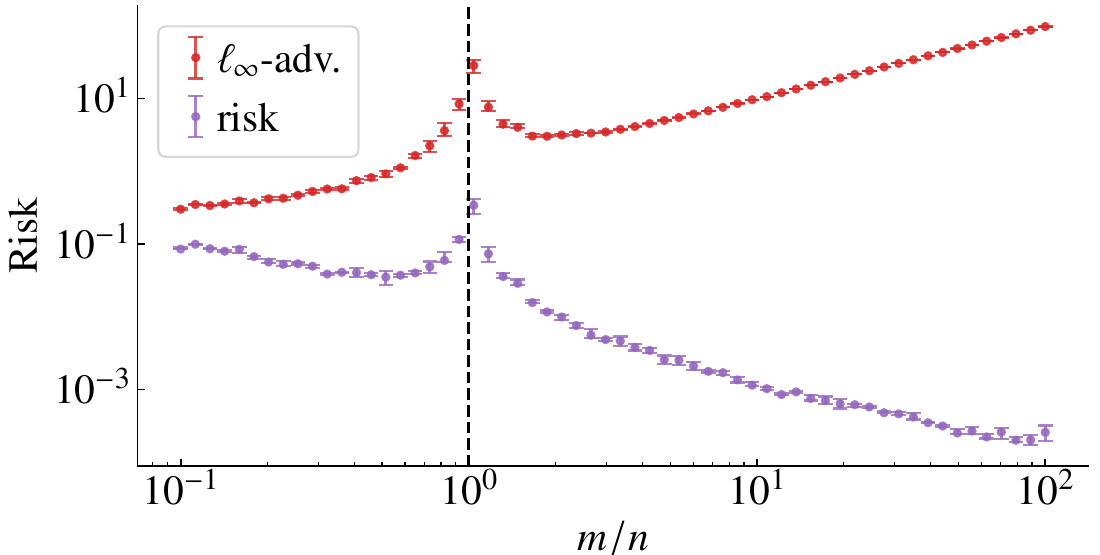}}
    \caption{\emph{Motivating example: weak features}. For the data generated as in~\eqref{eq:weak_features_problem}, we show the risk and the $\ell_\infty$-adversarial risk on the test dataset for: (a) the optimal predictor ${\mhat{\beta} = \left[\frac{1}{\sqrt{m}}, \dots, \frac{1}{\sqrt{m}}\right]}$; and, (b) the predictor obtained from a training dataset using the minimum-norm solution, i.e. Eq.~\eqref{eq:min-norm-sol}.}
    \label{fig:weak-features}
\end{figure}

In this section, we show brittleness and vulnerability to $\ell_\infty$-adversarial attacks arising from high-dimensionality.
\citet{tsipras_robustness_2019} makes use of a linear example to motivate their argument that robustness and accuracy might be at odds. We show a modified  construction\footnote{The construction is quite similar. The only differences are: 1) it is a regression problem rather than a classification problem; and, 2) we use $y/\eta$ rather than $y$ as the conditional mean of $x^j$ (in Eq.~\eqref{eq:weak_features_problem}). This is done to ensure that the signal-to-noise ratio is constant with $\eta$.}  to motivate how brittleness might appear in linear examples. The construction makes use of many features that are weakly correlated with the output. Let the input $x = (x^1, \cdots, x^m)$ and the output $y$ be normally distributed:
\begin{equation}
  \label{eq:weak_features_problem}
    y\sim \N(0, 1) \text{ and } x^j \mid y \sim \N\left(\frac{y}{\eta} ,  \frac{1}{\eta}\right) \text{ for } j = 1, \dots,  m.
\end{equation}
Following the choice of~\citet{tsipras_robustness_2019}, we use $\eta = \frac{1}{\sqrt{m}}$ implying that $\E{x}{\|x\|_2}$ remains constant with the number of features.
Here, we can show that the optimal predictor ${\mhat{\beta} = \left[\frac{1}{\sqrt{m}}, \cdots, \frac{1}{\sqrt{m}}\right]}$ results in a prediction $\hat{y} = \mhat{\beta}^\trnsp x$ that follows
the distributions ${(\mhat{\beta}^\trnsp x) | y \sim \N(y, \frac{1}{m})}$. This means that the prediction risk of this model is $R(\mhat{\beta}) = \E{x, y}{(y - \mhat{\beta}^\trnsp x)^2} = \frac{1}{m}$. Hence, the risk of our predictor goes to zero 
$R \rightarrow 0$ as the number of features goes to infinity, $m\rightarrow \infty$.

Here, $\|\mhat{\beta}\|_1 = \sqrt{m}$, implying that the $\ell_\infty$-adversarial risk would grow with a rate $\Omega(m)$ by Eq.~\eqref{eq:ineq_adversarial_risk}. This is one example where the risk goes to zero with the number of features \textit{at the same time} as the $\ell_\infty$-adversarial risk grows indefinitely. The risk and the $\ell_\infty$-adversarial risk for this example are displayed in Fig.~\ref{fig:weak-features}(a).

There are two aspects of this example that we will further refine in the coming sections. The first is that we are using an estimator that is obtained by minimizing the true risk, which is not a procedure that can be used in practice. In what follows, we show that a similar effect can be obtained if the empirical risk is  minimized (i.e., as in Fig. \ref{fig:weak-features}(b)). The second aspect is the scaling $\eta$, we will motivate different choices and show how they can yield quite different results.

\subsection{Preliminaries}

Here we focus on the minimum norm solution, which is often used when studying the behavior of overparameterized models in connection with the double-descent phenomenon~\citep{belkin_reconciling_2019}.  We will assume that the training and test data have been generated linearly with additive noise:
\begin{align}
    \label{eq:linear-data-model}
    (x_i, \epsilon_i) \sim P_x\times P_\epsilon,\qquad y_i = x_i^\trnsp \beta + \epsilon_i,
\end{align}
where $P_\epsilon$ is a distribution in $\R$ such that $\Exp{\epsilon_i} = 0$ and  $\Var{\epsilon_i} = \sigma^2$ and $\epsilon_i$ is assumed to be independent of~$x_i$. Moreover, $\Exp{x_i} = 0$ and $\Cov{x_i} = \Sigma$. The $\ell_2$-norm of the data generation parameter is denoted by $\|\beta\|^2_2 = r^2$.

Let $X\in\R^{n \times m}$ denote the matrix consisting of stacked training inputs $x_i^\trnsp$ and similarly let $y\in\R^{n}$ denote the output vector. The parameters are estimated as
\begin{equation}
  \label{eq:min-norm-sol}
  \mhat{\beta} =  (X^\trnsp X)^\pinv X^\trnsp y,
\end{equation}
where $ (X^\trnsp X)^\pinv$ represents the pseudo-inverse of $X^\trnsp X$. In the underparameterized  ($m < n$) region, it corresponds to the least-square solution.
In the overparameterized ($m > n$) case, where more than one solution is possible, this corresponds to the solution for which the parameter norm $\|\mhat{\beta}\|_2$  is minimum, the \textbf{minimum-norm solution}.

From Eq.~\eqref{eq:linear-data-model} and Eq.~\eqref{eq:min-norm-sol} it follows that:
\begin{equation}
  \label{eq:beta-closeform}
    \mhat{\beta} = \underbrace{(X^\trnsp X)^\pinv X^\trnsp X}_{\Phi} \beta + \underbrace{(X^\trnsp X)^\pinv}_{n\hat{\Sigma}^\pinv} X^\trnsp \epsilon,
  \end{equation}
where we have introduced the following notation $\hat{\Sigma} = \frac{1}{n}X^\trnsp X$,  $\Phi = \hat{\Sigma}^\pinv \hat{\Sigma}$ and $\Pi = I - \Phi$. Here, $\Phi$ and $\Pi$ are orthogonal projectors: $\Pi$ is the projection into the null space of $X$ and $\Phi$ into the row space of $X$. The first term in Eq.~\eqref{eq:beta-closeform} can be understood as a projection of the original parameter into the row space of the regressors and it is the parameters estimated in a noiseless scenario. The second term is the consequence of the noise. It follows that the risk and the expected parameter norm can be decomposed as in the subsequent Lemma. The proof is provided in the Supplementary Material.
\begin{lemma}[Bias-variance decomposition]
    \label{thm:bias-variance-decomposition} Denote  $\| z\| _{\Sigma}^ 2  =  z^ \trnsp \Sigma z$.  The expected risk and $\ell_2$ parameter norm are
\begin{eqnarray}
    \label{eq:bias-variance-decomposition}
    \E{\epsilon}{R(\beta)} &=& \| \Pi \beta\| _{\Sigma}^ 2 +  \frac{\sigma^2}{n} \tr(\hat{\Sigma}^\pinv \Sigma) + \sigma^2, \\
    \label{eq:l2norm-decomposition}
    \E{\epsilon}{\|\mhat{\beta}\|_2^2} &=& \| \Phi \beta\|_2^ 2 +  \frac{\sigma^2}{n} \tr(\hat{\Sigma}^\pinv).
\end{eqnarray}
\end{lemma}

The following bounds can be used in conjunction with Lemma~\ref{thm:bias-variance-decomposition} to analyze the adversarial risk
\begin{equation}
    \label{eq:bounds-linear-regression}
    R + \delta^2 L_q \le R_p^{\text{adv}} \le \left(\sqrt{R} + \delta  \sqrt{L_q} \right)^2,
  \end{equation}
where $L_q = \E{\epsilon}{\|\mhat{\beta}\|_q^2}$,  $R = \E{\epsilon}{R(\mhat{\beta})}$ and $R_p^{\text{adv}} = \E{\epsilon}{R_p^{\text{adv}}(\mhat{\beta})}$.

\begin{figure}
    \centering
    \includegraphics[width=0.45\textwidth]{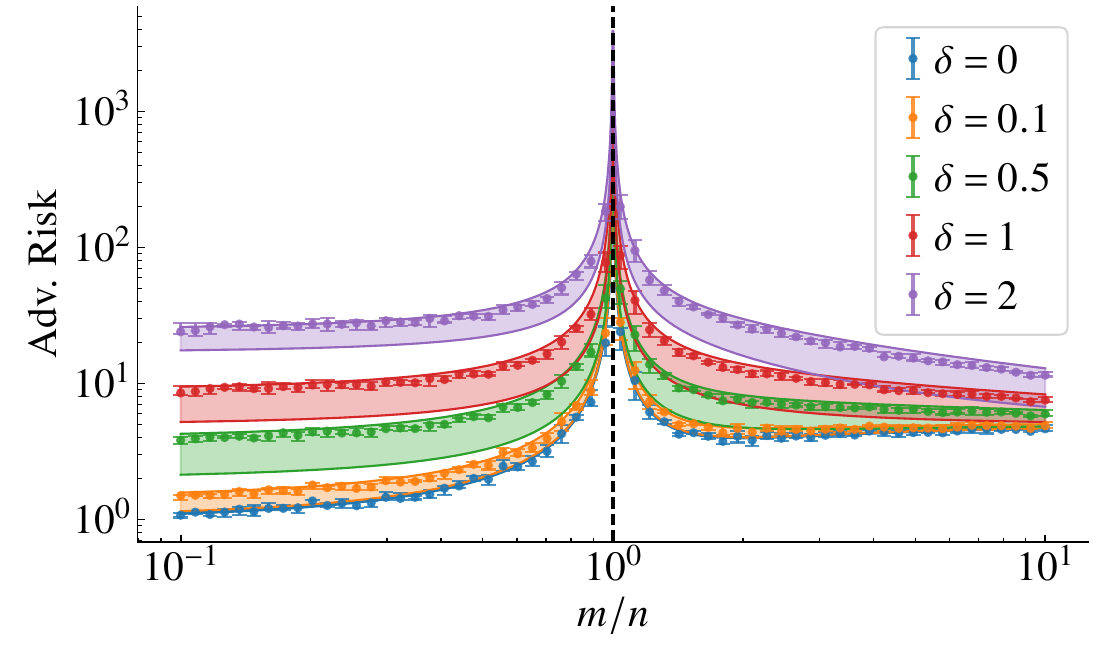}
    \caption{\emph{Adversarial $\ell_2$-risk, isotropic features.} The solid line shows the upper and lower bounds on the asymptotic risk. The results are for isotropic features with $r^2 =2$, $\sigma^2 = 1$.  The error bars give the median and the 0.25 and 0.75 quantiles obtained from numerical experiments (10 realizations) with a fixed training dataset of size $n=300$.}
    \label{fig:double-descent-l2-isotropic}
\end{figure}

\subsection{Isotropic feature model}
\label{sec:isotropic-features-model}
Let us now analyse the case of isotropic features, i.e.  we assume that the input $x_i$ has zero mean and unit variance ($\Sigma = I_m$).  Despite its simplicity, this model captures some interesting aspects of overparametrization, for instance, it  produces double-descent curves~\citep{hastie_surprises_2019}. It is used as the starting point for our study, which later is enriched by the analysis of a latent space model---in Section~\ref{sec:latent-space-model}---and of an equicorrelated features model---in Appendix~\ref{sec:equicorrelated-model}.
For this case, results from random matrix theory can be used to establish asymptotics for the terms in Eqs.~\eqref{eq:bias-variance-decomposition} and~\eqref{eq:l2norm-decomposition}. 

\begin{lemma}[Isotropic features, {\citet{hastie_surprises_2019}}]
\label{thm:asymptotics-linear-regression} 
For isotropic features with a moment of order greater than 4 that is finite. As $m, n\rightarrow \infty$, $m/n \rightarrow \gamma$, almost surely
\begin{eqnarray}
    \E{\epsilon}{R(\mhat{\beta})} &\rightarrow&
    \begin{cases}
    \sigma^2 \frac{\gamma}{1 - \gamma} + \sigma^2& \gamma < 1, \\
    r^2 (1 - \frac{1}{\gamma}) + \sigma^2 \frac{1}{\gamma - 1}+ \sigma^2 & \gamma > 1.
    \end{cases}\\
    \label{eq:asymptotic_l2_isotropic}
      \E{\epsilon}{\|\mhat{\beta}\|_2^2} &\rightarrow&
    \begin{cases}
    r^2 + \sigma^2 \frac{\gamma}{1 - \gamma} & \gamma < 1, \\
    r^2 \frac{1}{\gamma} + \sigma^2 \frac{1}{\gamma - 1} &  \gamma > 1.
    \end{cases}
\end{eqnarray}
\end{lemma}

Both terms  do have an asymptotic behavior that depends on the ratio  $\gamma$  between the number of features $m$ and the number of training datapoints $n$. In Fig.~\ref{fig:double-descent-l2-isotropic}, we illustrate one example of how the bounds in~\eqref{eq:bounds-linear-regression} can be combined with the asymptotic results above to obtain  asymptotic lower and upper bounds. The results obtained from the experiments closely follow these asymptotic bounds.

Let us already now informally point out that the second part of Lemma~\ref{thm:asymptotics-linear-regression} states that for a sufficiently large problem, in the overparameterized region, 
\begin{equation}
    \label{eq:approx_norm_beta}
     \|\hat\beta\|^2_2 \approx r^2\frac{1}{m/n} + \sigma^2 \frac{1}{m/n -1}.
   \end{equation}
We formalize the notion in Section~\ref{sec:non-asymptotic} using a concentration of measure.
It follows from it that even for a fixed signal magnitude, the norm of the estimated parameter decays with the number of parameters for overparameterized problems, i.e., $\|\hat\beta\|_2 = \bigO\left((\frac{m}{n})^{-\frac{1}{2}}\right)$. The  model becomes \emph{`smoother'} as more parameters are added to the model. This naturally yields models more robust to $\ell_2$ perturbations.  This can be observed in Fig.~\ref{fig:double-descent-l2-isotropic}: after the local minima $\gamma =  \frac{r}{r+ \sigma}$ in the standard risk, the risk is increasing with $\gamma$. However, the adversarial risk for, say, $\delta = 2$ is decreasing due to the tendency of the minimum-norm solution to select  smoother solutions. Moreover, while the standard risk does not have better results in the overparameterized region than in the underparameterized region, the adversarial risk in the overparameterized region can actually be better than the adversarial risk in the underparameterized region.

\subsection{Non-asymptotic results for the parameter $\ell_2$-norm}
\label{sec:non-asymptotic}

Central to our analysis of $\ell_2$-adversarial attacks is the idea that the parameter norm decays with the rate $\bigO\left((\frac{m}{n})^{-\frac{1}{2}}\right)$ even when the data generator parameter remains constant, i.e. $\|\beta\|_2 = r$. This intuition is formalized in the theorem below. We do not attempt to provide the most general result, instead, our choice is motivated by the fact that many of the steps in proving this theorem can be carried on to the $\ell_1$-norm.

\begin{theorem}
  \label{thm:rate-beta}
  Let the data be generated according to  Eq.~\eqref{eq:linear-data-model}. Assume additionally that $m > n$ and:
  \begin{enumerate}
  \item The noise $\epsilon$ and the regressor $x$ are sub-Gaussian.
  \item the regressor $x$ is sampled from a rotationally invariant distribution.
  \end{enumerate}
  Then there exists constants $c, C > 0$ such that for all $t < \sqrt{\frac{m}{n}}$ with probability ${1 - 2\exp(- c t^2 n)}$ we have:
  \begin{equation*}
     \|\mhat{\beta}\|_2 \le r \frac{1 + t}{\sqrt{\nicefrac{m}{n}}}  + \sigma \frac{1 + t}{C \sqrt{\nicefrac{m}{n}} - 1}.
  \end{equation*}
\end{theorem}

This theorem provides a non-asymptotic result and an exponential rate of convergence. It strengthens the assumptions from the previous section in two ways. First, it assumes the variables are sub-Gaussian, which is used to obtain an exponential rate of convergence. If this assumption is relaxed lower rates of concentration are obtained. For instance,~\citet{hastie_surprises_2019} does not assume this, which results in a convergence rate of $n^{-1/7}$. 

Secondly, it assumes the regressor to be rotationally invariant. Hence, given an orthogonal matrix $Q$, multiplication by this matrix does not change the distribution, i.e. $x\sim Qx$. Standard examples where~$x$ is rotationally invariant are values sampled from standard Gaussian or from the uniform distribution over the sphere. Rotational invariance implies isotropy, but not all isotropic distributions are rotationally invariant. Hence, this is again a stronger assumption.

From Lemma~\ref{thm:bias-variance-decomposition} we have:
\begin{equation*}
    \E{\epsilon}{\|\mhat{\beta}\|_2^2} = \|\Phi \beta\|^ 2_2 +  \frac{\sigma^2}{n} \tr(\hat{\Sigma}^\pinv).
  \end{equation*}

The next Lemma gives concentration inequalities for the eigenvalues of $\hat{\Sigma}^\pinv$. If we use $\lambda_i(\hat{\Sigma}^\pinv)$ to denote the $i$-th eigenvalue of $\hat{\Sigma}^\pinv$, we have that $\tr(\hat{\Sigma}^\pinv) =  \frac{1}{n}\sum_i \lambda_i(\hat{\Sigma}^\pinv)$. Hence, the following result immediately implies that the second term of the above expression concentrates around $\Theta((\frac{m}{n})^{-1})$. The proof is provided in the Supplementary Material.
\begin{lemma}
  \label{thm:bounds-trace}
  Let $x_i\in \R^{m} $ be independently sampled sub-Gaussian random vectors, $\hat{\Sigma} = \frac{1}{n} \sum_{i=1}^n x_i x_i^\trnsp$, and let $m  > n$. Then there exists a constant $C>0$ such that, with probability greater than $1 - 2 \exp(-m)$,
\begin{equation*}
   \frac{1}{\left(C \sqrt{\frac{m}{n}}+1\right)^2} \le \lambda_i(\hat{\Sigma}^\pinv) \le \frac{1}{\left(C\sqrt{\frac{m}{n}} - 1\right)^2}.
\end{equation*}
\end{lemma}

\begin{proof}
From the lemma statement: $x_i$ are independently sampled sub-Gaussian vectors. Let $X$ be a matrix containing the vectors  $x_i$  as its rows.
Let $s_i(X)$ denote the $i$-th sigular value of $X$.
From~\citet[Theorem 4.6.1]{vershynin_high-dimensional_2018} we have that there exist a constant $C$ such that with probability larger then $1- 2 \exp(- t^2)$,
\begin{equation*}
   \sqrt{n} - C \left(\sqrt{m} + t\right) \le s_i(X) \le  \sqrt{n} + C \left(\sqrt{m} + t\right), \forall i = 1, \dots, n
 \end{equation*}
Set $t = \sqrt{m}$ then, since $\lambda_i(\hat{\Sigma}) = \left(\frac{1}{\sqrt{n}} s_{(n-i)}(X)\right)^{-2}$ we obtain with probability greater than $1- 2 \exp(-m)$ that
\begin{equation*}
   \frac{1}{(1 + C \left(\sqrt{\frac{m}{n}}\right))^2}\le \lambda_i(\hat{\Sigma}) \le  \frac{1}{(1 - C \left(\sqrt{\frac{m}{n}} \right))^2}, \,  \forall i = 1, \dots, n.
 \end{equation*}
Finally, since $m > n$, we have that  $1- 2 \exp(-m) >  1- 2 \exp(-n)$ and the result follows.
\end{proof}

Next, we turn to the analysis of the first term in Eq.~\eqref{eq:l2norm-decomposition}. In the case $m>n$, $\Phi$ is a projection matrix that projects a vector from $\R^m$ into a subspace of dimension $n$.
The set of all possible subspaces of dimension $n$ in $\R^m$ is well studied and known as the Grassmannian manifold $G(m, n)$. 
There is a one-to-one relationship between the projection matrices $\Phi$ and the points in this manifold.
For the case when $X$ is rotationally invariant, we have, given any orthogonal matrix $Q$ that
\begin{equation*}
    \Phi = (X^\trnsp X)^\pinv(X^\trnsp X) \sim Q(X^\trnsp X)^\pinv(X^\trnsp X)Q^\trnsp = Q\Phi Q^\trnsp.
  \end{equation*}
Hence, the subspace is invariant to rotation and it is possible to establish that the matrix $\Phi$ is a random projection that projects into a subspace sampled uniformly (i.e., Haar measure) from the Grassmannian $G(m, n)$. The next result is from~\cite[Lemma 5.3.2]{vershynin_high-dimensional_2018} and it states that the norm $\|\Phi \beta\|_2$ of the projection of $\beta$ into this $n$ dimensional subspace concentrates around $\sqrt{\frac{n}{m}} \|\beta\|$.

\begin{lemma}[{{\citet[Lemma 5.3.2]{vershynin_high-dimensional_2018}}}]
  \label{thm:bounds-norm}
 Let $\beta\in \R^m$ be a vector and $\Phi\in \R^{m \times m}$ be a projection from $\R^m$ onto a random $n$-dimensional subspace uniformly sampled from $G(m, n)$. Then,
\begin{enumerate}
    \item $\E{\Phi}{\|\Phi\beta\|_2^2} = \frac{n}{m}  \|\beta\|_2^2$;
    \item There exist a constant $c$, such that with probability greater then $1 - 2 \exp(-ct^2 n)$, we have:
    \begin{equation}
        (1-t)\sqrt{\frac{n}{m}} \|\beta\|_2 \le \|\Phi\beta\|_2\le (1+t) \sqrt{\frac{n}{m}}  \|\beta\|_2
    \end{equation}
\end{enumerate}
\end{lemma}

The following proposition will also be needed.
\begin{proposition}
      \label{thm:concentration_around_proj_norm}
  We have that:
  \begin{equation}
    \left|\|\mhat{\beta}\|_2 - \|\Phi \beta\|_2\right| \le \frac{1}{\sqrt{n}} \sqrt{\|\hat{\Sigma}^\pinv\|_2}\|\epsilon\|_2.
  \end{equation}
\end{proposition}
\begin{proof}
  From Eq.~\eqref{eq:beta-closeform} and the triangular inequality:
\begin{equation}
  \left| \|\mhat{\beta}\|_2 - \|\Phi \beta\|_2 \right|\le  \frac{1}{\sqrt{n}} \left\|\frac{1}{\sqrt{n}}\hat{\Sigma}^\pinv X^\trnsp \epsilon\right\|_2
\end{equation}
In turn, we have  that
\begin{equation*}
  \left\|\frac{1}{n}\hat{\Sigma}^\pinv X^\trnsp \epsilon\right\|_2 \le \left\|\frac{1}{\sqrt{n}}\hat{\Sigma}^\pinv X^\trnsp \right\|_2\| \epsilon\|_2.
\end{equation*}
Here, $\|\cdot\|_2$ is used to denote the operator norm, such that for a matrix $A\in \R^ {m \times n}$, $\|A\|_2 = \max_i \sqrt{\lambda_{i}(A^\trnsp A)}$. Hence:
\begin{equation*}
   \left\|\frac{1}{n}\hat{\Sigma}^\pinv X^\trnsp \right\|_2 =  \max_i \sqrt{\lambda_i(\hat{\Sigma}^\pinv \hat{\Sigma}  \hat{\Sigma}^\pinv)} =   \sqrt{\|\hat{\Sigma}^\pinv\|_2},
\end{equation*}
where the second equality follows from direct use of the property $ \hat{\Sigma}^\pinv \hat{\Sigma}  \hat{\Sigma}^\pinv =  \hat{\Sigma}^\pinv$ and the fact that  $\hat{\Sigma}^\pinv$ is positive semidefinite.
\end{proof}
Equipped with the proposition and the lemmas we are now ready to prove the theorem. Note that  Lemma~\ref{thm:bias-variance-decomposition} could provide another possible route to prove similar results and (maybe) tighter bounds.
Nonetheless, here, we choose to use Proposition~\ref{thm:concentration_around_proj_norm} for two reasons: 1) it can  easily be combined with non-asymptotic results; 2) the argument applied above can be extended for any other $p$-norms.  We also point out that an analogous procedure could be used to provide a lower bound on $\|\mhat{\beta}\|_2$.
\begin{proof}[Proof of Theorem~\ref{thm:rate-beta}]
  From Proposition~\ref{thm:concentration_around_proj_norm} we have $\|\mhat{\beta}\|_2   \le \|\Phi \beta\|_2 + \frac{\|\epsilon\|_2}{\sqrt{n}} \sqrt{\lambda_{\max}(\hat{\Sigma}^\pinv)}$. Due to the fact that the noise is sub-Gaussian, a straightforward application of Theorem~3.1.1 from~\cite{vershynin_high-dimensional_2018} implies that $\left|\frac{\|\epsilon\|_2}{\sqrt{n}} - \sigma^2\right| \le t$ with probability greater  than $1 - 2\exp(c t^2 n)$. This together with Lemma~\ref{thm:bounds-trace} yields the desired upper bound for the second term. The first term can be bounded using Lemma~\ref{thm:bounds-norm}. The result follows.
\end{proof}

\subsection{$\ell_p$ adversaries}
\label{sec:other-ell-pnorms}

\begin{figure}
  \centering
  
    \subfloat[Risk]{\includegraphics[width=0.45\textwidth]{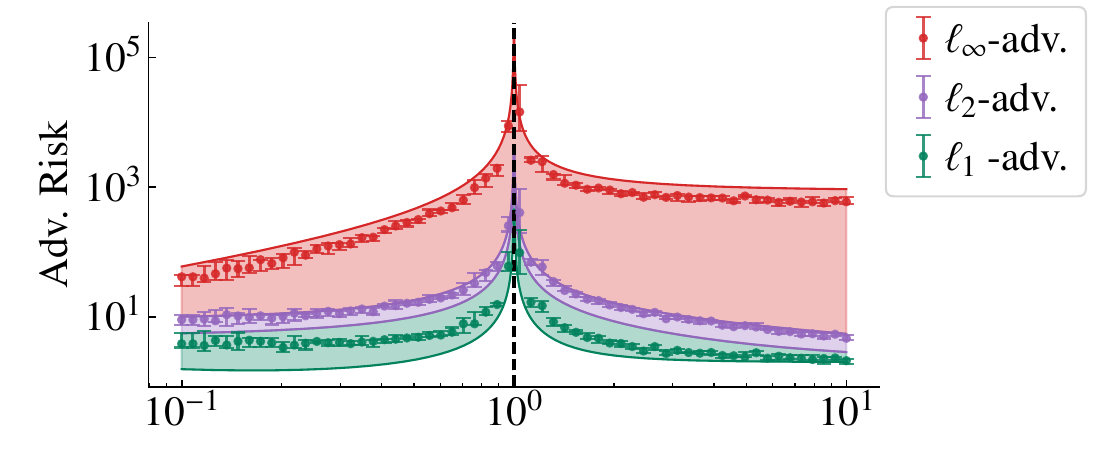}}\\
    \subfloat[Parameter norm]{\includegraphics[width=0.45\textwidth]{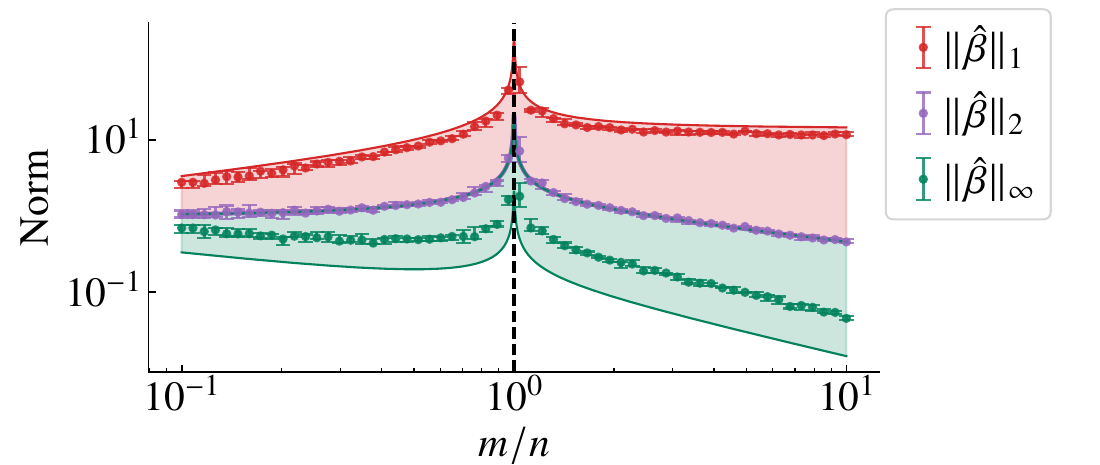}}
    \caption{\emph{Adversarial $\ell_p$-risk, isotropic  features.} The error bars give the median and the 0.25 and 0.75 quantiles obtained from numerical experiments (10 realizations) with a fixed training dataset of size $n=100$ and for adversarial disturbances of magnitude $\delta = 2$. The shaded region in \textcolor{MyBlue}{blue} gives the upper and lower bounds for $p=2$. The bounds obtained from norm inequalities~\eqref{eq:relation-between-p-norm}  for  $p, q \in \{1, \infty\}$ are given in \textcolor{MyRed}{red} and \textcolor{MyGreen}{green}. The results are for $r^2 =1$, $\sigma^2 = 1$.}
    \label{fig:double-descent-lp}
  \end{figure}

We now turn to the study of $\ell_p$-adversarial attacks when $p \not= 2$. As in Eq.~\eqref{eq:bounds-linear-regression}, let $q$ be the complement of $p$.  The following well-known relationship between vector norms will be useful in our development.

\begin{lemma}[Relationship between vector $p$-norms]
\label{thm:relation-between-p-norm}
Let $p$ and $q$ be values in the range $[1, \infty]$ and $\mhat{\beta}\in\R^m$. Assume that $q > p$, then:
\begin{equation}
    \label{eq:relation-between-p-norm}
    \|\mhat{\beta}\|_q \le  \|\mhat{\beta}\|_p \le  m^{1/p - 1/q}\|\mhat{\beta}\|_q.
\end{equation}
\end{lemma}
The leftmost inequality follows from an application of Minkowski's inequality and the rightmost from an application of H\"older's inequality.

 The asymptotic results from Lemma~\ref{thm:asymptotics-linear-regression} to compute $R$ and $L_2$ can now be used in conjunction with the above inequalities to find the upper and lower bounds on the adversarial risk.
Hence, for any $1 \le p < 2$, the upper bound is the same as the upper bound obtained for $\ell_2$ attacks. However, there is a new multiplicative term in the lower bound. For instance, the $\ell_1$ adversarial risk is bounded by
\begin{equation}
    \label{eq:bounds-linear-regression_l1}
    R(\mhat{\beta}) + \frac{\delta^2}{m}\|\mhat{\beta}\|_2^2 \le R_1^{\text{adv}}(\mhat{\beta}) \le \left(\sqrt{R(\mhat{\beta})} + \delta  \|\mhat{\beta}\|_2 \right)^2.
\end{equation}
On the other hand, for $\ell_p$-adversarial attacks with $p>2$, we obtain an asymptotic upper bound that grows with $m^{1 - \frac{2}{p}}$. As an example, for $\ell_\infty$ attacks,
\begin{equation}
    \label{eq:bounds-linear-regression_linfty}
    R(\mhat{\beta}) + \delta^2 \|\mhat{\beta}\|_2^2 \le R_\infty^{\text{adv}}(\mhat{\beta}) \le \left(\sqrt{R} + \delta  \sqrt{m }\|\mhat{\beta}\|_2\right)^2.
\end{equation}
In Fig.~\ref{fig:double-descent-lp}(a) we illustrate the bounds obtained in this way.
We note that the  $\ell_\infty$-adversarial risk follows  the upper bound closely.
Moreover, Lemma~\ref{thm:relation-between-p-norm} implies that
\begin{equation}
  \|\hat\beta\|_2 \le \|\hat\beta\|_1 \le \sqrt{m}  \|\hat\beta\|_2 .
\end{equation}
From Fig.~\ref{fig:double-descent-lp}(b) we see that the $\ell_1$-norm of the estimated parameter seems to follow the upper bound closely. At the same time, the adversarial risk is also close to the upper bound closely. Next, we provide some insight into this observation, by following the same steps used in the non-asymptotic analysis of the $\ell_2$ parameter norm.

Lemma~\ref{thm:bounds-norm} show how  $\|\Phi\beta\|_2$  concentrate around $\sqrt{\frac{n}{m}}\|\beta\|_2$. One might wonder whether similar concentration inequalities can be obtained also for the $\ell_1$-norm.  In Fig.~\ref{fig:conjecture-l1} we illustrate the experiments for both $\|\Phi\beta\|_1$ and $\|\Phi\beta\|_2$. The first plot just illustrates the results known for the $\ell_2$-norm  from Lemma~\ref{thm:bounds-norm}; the second plot suggests that the $\ell_1$-norm of the projection has mean $c\sqrt{n} \|\beta\|_2$. From the experiments, we also estimate that $c\approx 0.8$. We state this result as a conjecture.
\begin{conjecture}
  \label{thm:bounds-norml1}
Let $\beta\in \R$ and $\Phi$ be a projection from $\R^m$ onto a random $n$-dimensional subspace uniformly sampled from the Grassmannian manifold $G(m, n)$.  Then,
$\E{\Phi}{\|\Phi \beta\|_1} = c\sqrt{n} \|\beta\|_2$.
\end{conjecture}
Since  $\|\Phi\beta\|_1 $ concentrates around its mean, the conjecture also implies a high probability statement.  Indeed, in  Appendix~\ref{sec:additional-proofs} it  is  proved that with probability greater than $1 - \exp(-t^2 n)$
\begin{equation} 
\label{thm:concentration-around-the-mean-norml1}
    |\|\Phi\beta\|_1 - \E{\Phi}{\|\Phi \beta\|_1} | < t \sqrt{n}\|\beta\|_2.
\end{equation}
Combined with this result, the conjecture implies that with probability greater than  $1 - \exp(-t^2 n)$,
\begin{equation}
    |\|\Phi\beta\|_1 -  c\sqrt{n} \|\beta\|_2 | < t \sqrt{n} \|\beta\|_2.
\end{equation}

\noindent
We point out that this result does have important consequences for the study of overparameterized models. We obtain smoother models by increasing the number of parameters for the $\ell_2$-norm, but the conjecture implies that this does not happen for the $\ell_1$-norm.
Indeed, it implies that with high probability 
\begin{equation}
 (c - t)\sqrt{m} \|\beta\|_2 < \frac{\|\Phi\beta\|_1}{\|\Phi\beta\|_2} <  (c + t)\sqrt{m} \|\beta\|_2.
\end{equation}
Now, using exactly the same argument as in Lemma~\ref{thm:concentration_around_proj_norm}, we obtain
\begin{equation*}
  \|\Phi \beta\|_1 - \|\epsilon\|_2 \sqrt{\tfrac{m}{n}\|\hat{\Sigma}^\pinv\|_2} \le  \|\mhat{\beta}\|_1 \le  \|\Phi \beta\|_1 +  \|\epsilon\|_2 \sqrt{\tfrac{m}{n}\|\hat{\Sigma}^\pinv\|_2} 
\end{equation*}
The conjecture implies that $ \|\Phi \beta\|_1 = \Theta(\sqrt{n}) \|\beta\|_2$, and Lemma~\ref{thm:bounds-trace} implies that the second term is $\Theta(\sqrt{n})  \|\epsilon\|_2$. Hence, for a sufficiently large signal-to-noise ratio ($r > \sigma$) it follows from the conjecture that $\|\mhat{\beta}\|_1 = \Theta(\sqrt{n})$. Since we obtained in the Theorem~\ref{thm:rate-beta} that $\|\mhat{\beta}\|_2 = \Theta(\sqrt{\frac{n}{m}})$, it follows that $\|\mhat{\beta}\|_1 =  \Theta(\sqrt{m})\|\mhat{\beta}\|_2$, which is consistent with the results we are experimentally observing.

 \begin{figure}
    \centering
    \subfloat[$\ell_2$-norm of projection]{\includegraphics[width=0.4\textwidth]{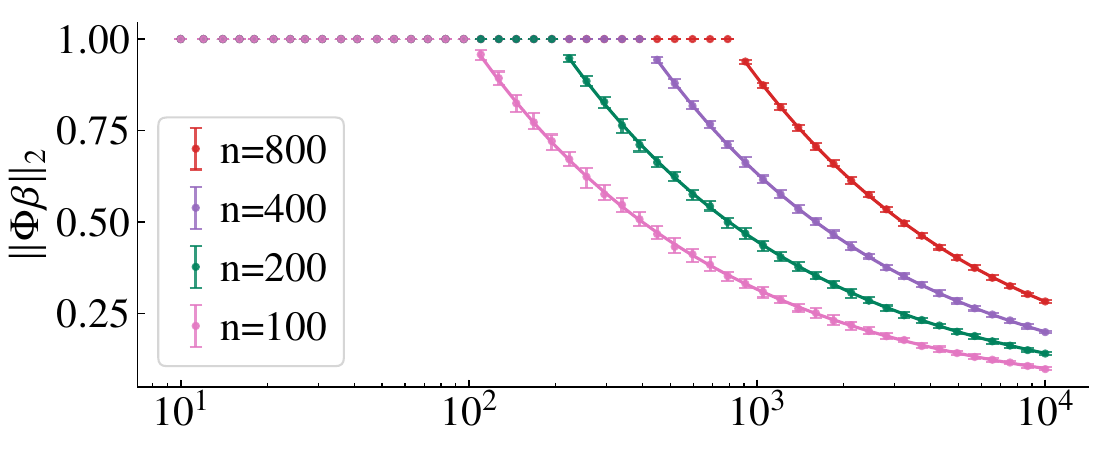}}\\
    \subfloat[$\ell_1$  norm of projection]{\includegraphics[width=0.4\textwidth]{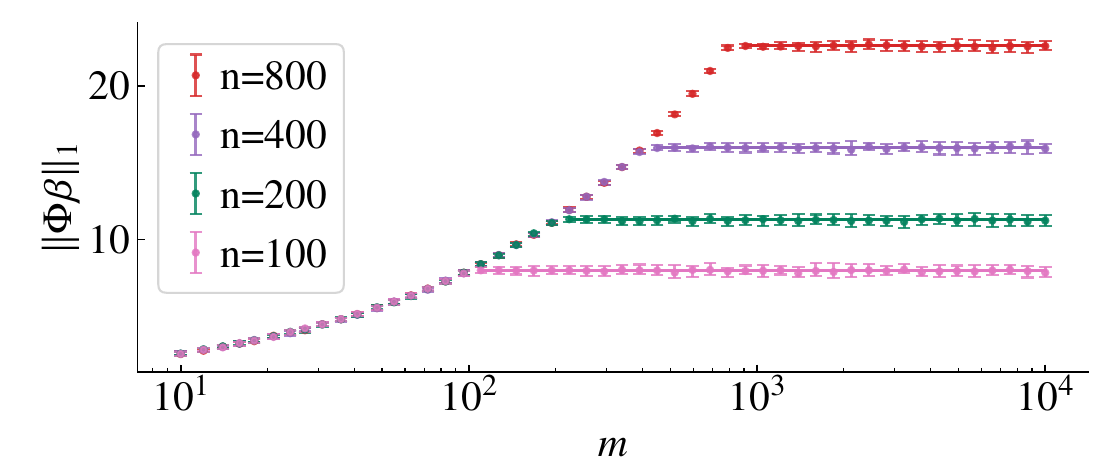}}
    \caption{\emph{Random projection and norms}. Let $\Phi$ be a  (uniform) random projection from $R^n$ into a subspace of dimension $m$. The full lines give the values predicted when $m<n$: $(m/n)^{-1/2}$ for the $\ell_2$-norm; and, the constant rate $c \sqrt{n}$ for $c = 0.8$ for the $\ell_1$-norm. The error bars give the median and interquartile range of $100$ experiments.}
    \label{fig:conjecture-l1}
\end{figure}

\subsection{Scaling}
\label{sec:scaling}

The scaling of variables plays an important role in the analysis. Assume that a given $\mhat{\beta}$ was estimated and  that the corresponding model prediction is $\mhat{\beta}^\trnsp x$. 
By simply redefining the input variable as $\tilde{x} = \frac{1}{\eta} x$ we could obtain an equivalent model  $\tilde{\beta}^\trnsp\tilde{x}$ that, for $\tilde{\beta} = \eta \mhat{\beta}$, would yield exactly the same predictions. 

Notice that while the standard risk $R$ for this new, rescaled, model is exactly the same as the first, the norm of the estimated parameter $\|\tilde{\beta}\|_q$ is $\eta$ times larger.
The adversarial risk is not the same for the two models, as an inspection of Eq.~\eqref{eq:bounds-linear-regression} reveals. The difference is because the \textit{relative}  magnitude of the adversarial disturbance is larger in the second model (even though it is the same in absolute value).

Since we are interested in the impact that the number of parameters~$m$ has on adversarial robustness, we will let the scaling factor depend on this parameter, i.e. $\eta = \eta(m)$.
The next proposition motivates two choices of scaling.  The proof is provided in the Supplementary Material.
\begin{proposition}
  \label{thm:scaling-motiv}
  Let $x$ be an isotropic random vector, $\Exp{\|x\|_2^2} = m$. Additionally, if $x$ is a sub-Gaussian random vector, then $\Exp{\|x\|_\infty} = \Theta(\sqrt{\log(m)})$.
\end{proposition}

Hence, $\eta(m) = \sqrt{m}$ or $\eta(m) = \sqrt{\log{m}}$ are both quite natural choices of the scaling factor. They render, respectively, the expected $\ell_2$ and $\ell_\infty$-normss of the input vector constant as the number of features~$m$ varies. 

Assume that the inputs are redefined as $\tilde{x}_i =  \frac{1}{\eta(m)} x_i$. A quick inspection of Eq.~\eqref{eq:min-norm-sol} reveals that the estimated parameter is $\tilde{\beta} =  \eta(m) \mhat{\beta}$.
The risk $R(\mhat{\beta})$ does not change by the transformation, but the expected squared norm of the parameter does, ${\|\tilde\beta\|_2 = \left(\eta(m)\right)^2 \|\hat\beta\|_2}$. Hence, when $\eta(m) = \sqrt{\log{m}}$, it follows from Eq.~\eqref{eq:approx_norm_beta} that
\begin{equation}
  \label{eq:l2_norm_logsqrt_scaling}
     \|\tilde\beta\|_2^2 \approx \log m \left(r^2\frac{1}{m/n} + \sigma^2 \frac{1}{m/n -1}\right).
\end{equation}
Here, the logarithmic term changes slowly compared to the linear term in the denominator.
Hence, the result is similar to what was obtained without any scaling. On the other hand, the square root scaling $\eta(m) = \sqrt{m}$ yields:
\begin{equation}
    \label{eq:l2_norm_sqrt_scaling}
     \|\tilde\beta\|_2^2 \approx r^2 n + \sigma^2 \frac{1}{1/n - 1/m}.
\end{equation}
Here, the parameter norm does not go to zero. Instead, it approaches a constant as $m\rightarrow \infty$. 
The behavior is illustrated in Fig.~\ref{fig:isotropic-predrisk-and-norm} in the Supplementary material. One interesting consequence of Eq.~\eqref{eq:l2_norm_sqrt_scaling} is that $\|\tilde\beta\|_2$ grows with the number of training datapoints. Hence, \textit{the $\ell_2$-adversarial performance degrades as we add more training data points}.

The situation is even more pathological in the case of $\ell_\infty$-adversarial attacks. In Fig.~\ref{fig:linf-isotropic},
we show the $\ell_\infty$-adversarial risk as a function of $m$ when the input is scaled by $\eta(m) = \sqrt{m}$ and $\eta(m) = \sqrt{\log(m)}$. We also provide the upper bound obtained from Lemma~\ref{thm:asymptotics-linear-regression} and the inequality in~\eqref{eq:bounds-linear-regression_linfty}.  The behavior of $\ell_\infty$-adversarial attacks is governed by $\|\tilde{\beta}\|_1$ (see Eq.~\eqref{eq:ineq_adversarial_risk}) and we have $\|\tilde{\beta}\|_1 =  \Theta(\sqrt{m})\|\tilde{\beta}\|_2$ (recall Section~\ref{sec:other-ell-pnorms}). Hence, \eqref{eq:l2_norm_sqrt_scaling} and \eqref{eq:l2_norm_logsqrt_scaling} yield, respectively, $\|\tilde{\beta}\|_1 = \Theta(\sqrt{m})$ and $\|\tilde{\beta}\|_1 = \Theta(\sqrt{\log m})$, which explain the behavior observed in the figure.

\begin{figure}
    \centering
    \includegraphics[width=0.45\textwidth]{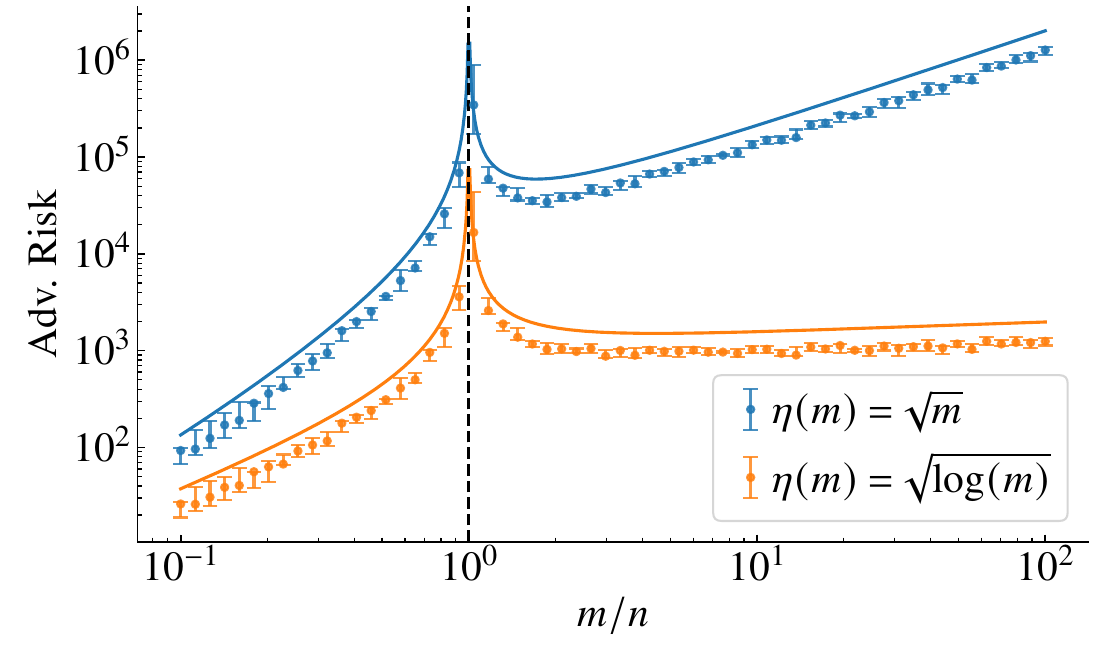}
    \caption{\emph{Adversarial $\ell_\infty$-risk, isotropic  features}.  The asymptotic upper bound is indicated by the full trace.  The error bars give the median and the 0.25 and 0.75 quantiles obtained from numerical experiments (10 realizations). The analysis is performed for a fixed training dataset of size $n=100$ and  for adversarial disturbances of magnitude $\delta = 0.1$. The results are for $r^2 =1$, $\sigma^2 = 1$. We show the results for two different scaling: $\eta(m) = \sqrt{m}$ and $\eta(m) = \sqrt{\log(m)}$.} 
    \label{fig:linf-isotropic}
\end{figure}

We end this section with another interpretation of rescaling. Let us consider the following change of variables $\tilde x = \frac{1}{\eta(m)} x $ and $ \tilde{\beta} = \eta(m) \mhat{\beta}$.  The next proposition states that this can also be interpreted as keeping the input and parameter constant while re-scaling the adversarial disturbance region $\delta$ by a factor $\eta(m)$. 
\begin{proposition}
Let 
$$\text{adv-error}(x, \beta, \delta) = \max_{\|\Delta x\|_2 \le \delta}(y - (x + \Delta x)^\trnsp\beta)^2,$$
then we have that:
$$\text{adv-error}\left(\frac{x}{\eta(m)}, \eta(m)\beta, \delta\right) =  \text{adv-error}\left(x, \beta, \eta(m)\delta\right).$$
\end{proposition}

\subsection{Latent space model}
\label{sec:latent-space-model}

For the model studied in the previous section, it is in general possible to achieve where the test error is smaller in the underparameterized region than in the overparameterized region. 
Thus, it could be argued that the lack of $\ell_\infty$-adversarial robustness in the overparameterized region should not be a problem in practice. 
Let us now illustrate a different data generation procedure for which we have better performance in the overparameterized regime and where the performance is continuously improved as more features are added. 
However, it is still possible to observe that the $\ell_\infty$-adversarial robustness degrades indefinitely with the number of features (recall Fig.~\ref{fig:three-types-of-risk}).

\begin{figure}
    \centering
    \subfloat[$\ell_2$ attacks]{\includegraphics[width=0.4\textwidth]{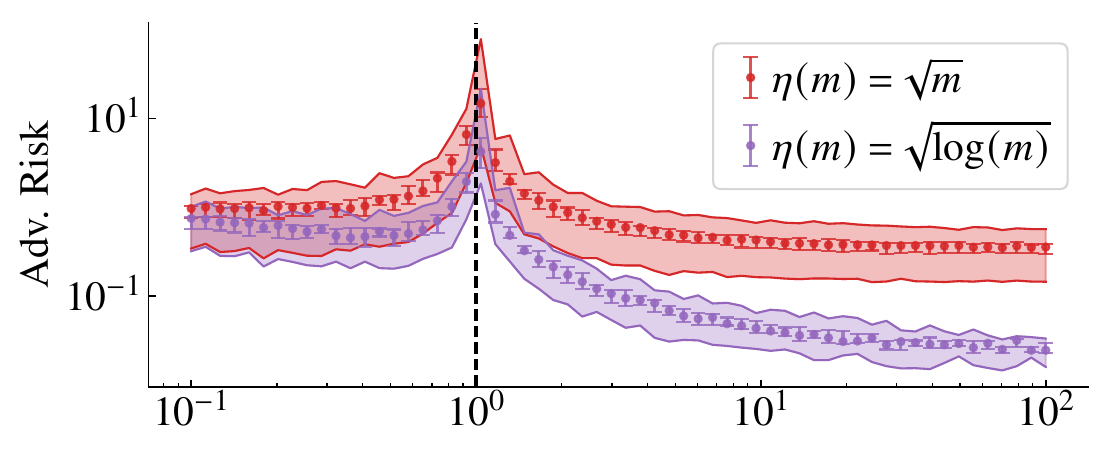}}\\
    \subfloat[$\ell_{\infty}$ attacks]{\includegraphics[width=0.4\textwidth]{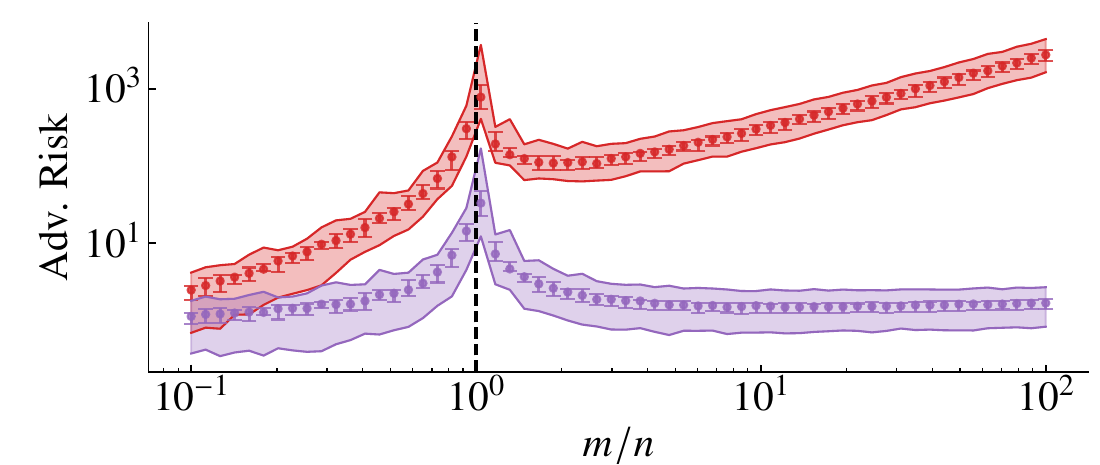}}
    \caption{\emph{Adversarial risk for a latent model}. The  median and interquartile range of the adversarial risk is obtained from numerical experiments (10 realizations) and indicated by error bars. Unlike the other plots, here the shaded region indicates the upper and lower bounds obtained empirically from Eq.~\eqref{eq:ineq_adversarial_risk}. The empirical risk in the test and in the parameter norm are obtained from the experiments. The analysis is performed for adversarial disturbances of magnitude $\delta = 0.1$ and for a training dataset with size $n=200$. The results are for $\sigma_\xi^2 = 0.1$ and for a latent space with constant dimension $d = 20$. We show the results for two different scaling: $\eta(m) = \sqrt{m}$ and $\eta(m) = \sqrt{\log(m)}$. }
    \label{fig:latent-l2-and-linf}
\end{figure}

We consider a data model where the features $x$ are noisy observations of a lower-dimensional subspace of dimension $d$.  A vector in this \textit{latent space} is represented by $z \in \R^d$. This vector is indirectly observed via the features $x \in \R^m$ according to
\begin{equation}
    \label{eq:latent-model-features}
    x = W z + u,
\end{equation}
where $W$ is an $m \times d$ matrix, for $m \ge d$.
We assume that the responses are described by a linear model in this latent space
\begin{equation}
    \label{eq:latent-model-outputs}
    y = \theta^\trnsp z + \xi,
\end{equation}
where $\xi \in \R$ and $u\in \R^m$ are mutually independent noise variables. Moreover, $\xi \sim \N(0, \sigma_{\xi}^2)$ and $u \sim \N\left(0, I_m\right)$. We consider the features in the latent space to be isotropic and normal $z_i = \N\left(0, I_d\right)$. To facilitate the analysis, we choose $W$ such that its columns are orthogonal, $W^\trnsp W = \frac{m}{d} I_d$, where the factor $ \frac{m}{d}$ is introduced to guarantee that the signal-to-noise ratio of the feature vector $x$ (i.e. $\frac{\|W z_i\|_2^2}{\|u_i\|_2^2}$) is kept constant.

This model is related to the other setups we have presented so far. The weak feature example (Section~\ref{sec:unidimensional-latent-space}) is a special case of this model class with $\theta = 1$, $\sigma_{\xi} = 0$, and $W~=~1$. Moreover, this latent model can actually be written as in Eq.~\eqref{eq:linear-data-model}, as we show in the Supplementary Material, Section~\ref{sec:latent-model-extension}.

Assume a training dataset $\{(x_i, y_i)\}_{i=1}^n$ generated using the above procedure. To this data, we fit a linear model $\mhat{\beta} x_i$ using the minimum-norm solution. Most of the arguments we presented for the isotropic case can be reused here. Asymptotics from~\citet{hastie_surprises_2019} are available in the Supplementary Material. From our non-asymptotic analysis in the isotropic case we obtained $\|\mhat{\beta}\|_2 = \bigO\left(\frac{1}{\sqrt{m/n}}\right)$, the same rate is obtained here. See  Supplementary Material, Section~\ref{sec:latent-model-extension}.

If again, we allow the input to be scaled, i.e. $\tilde{x}_i =  \frac{1}{\eta(m)} x_i$, a similar analysis shows that the factors $\eta(m) = \sqrt{m}$ and $\eta(m) = \sqrt{\log(m)}$ would correspond to keeping $\E{x}{\|x\|_2}$ and $\E{x}{\|x\|_\infty}$ constant, respectively. In Fig.~\ref{fig:latent-l2-and-linf} we illustrate the results for the two scaling and a fixed latent dimension $d = 20$.   We observe that the adversarial  $\ell_2$-risk, for both input scalings, continuously decreases in the overparameterized region and achieves  better results there than in the underparameterized region. The adversarial  $\ell_\infty$-risk,  on the other hand, presents quite a different behavior depending on the scaling. For $\eta(m) = \sqrt{m}$ it displays a linear growth with the number of parameters in the overparameterized region, while it remains basically constant when $\eta(m) = \sqrt{\log(m)}$.  Fig.~\ref{fig:three-types-of-risk} is an illustration of this same setting where we also include the standard risk in the same plot (scaling $\eta(m) =  \sqrt{m}$). Additional results are presented in Supplementary Material,  Section~\ref{sec:latent-model-extension}.

\section{Adversarial training and regularization}
\label{sec:adversarial-training}

Empirical risk minimization (ERM) is a popular paradigm for estimating predictive models~\cite{shalev-shwartz_understanding_2014}.
In the last section, the model was trained to minimize the empirical risk $\mhat{R}(\beta)$ but evaluated according to an adversarial criteria.
One natural idea to obtain models that are more robust to adversarial attacks is to instead minimize  the empirical adversarial risk,
\begin{equation}
  \label{eq:empirical_advrisk}
    \mhat{R}_p^{\text{adv}}(\mhat{\beta}) = \frac{1}{n}\sum_{i=1}^n{\max_{\|\Delta x_i\|_p \le \delta}(y_i - (x_i + \Delta x_i)^\trnsp\mhat{\beta})^2}.
  \end{equation}
This method is commonly called adversarial training~\citep{madry_deep_2018}. 

In this section, we use Lemma~\ref{thm:advrisk-closeform} to develop a convex formulation of adversarial training for linear regression problems. With this tool in hand, we explore the effect of adversarial training on how the model robustness changes the number of features. We also compare it to ridge regression,
\begin{equation}
  \label{eq:ridge}
    \mhat{R}_{\text{ridge}}(\mhat{\beta}) = \frac{1}{n}\sum_{i=1}^n{(y_i - x_i^\trnsp\mhat{\beta})^2} + \delta \|\mhat{\beta}\|_2^2,
  \end{equation}
  and lasso,
  \begin{equation}
  \label{eq:lasso}
    \mhat{R}_{\text{lasso}}(\mhat{\beta}) = \frac{1}{n}\sum_{i=1}^n{(y_i - x_i^\trnsp\mhat{\beta})^2} + \delta \|\mhat{\beta}\|_1.
  \end{equation}

\subsection{Adversarial training using convex programming}
\label{sec:adversarial-train-convex}
Using Lemma~\ref{thm:advrisk-closeform} we can show the convexity of the adversarial risk (defined in Eq.~\eqref{eq:adversarial-risk}).

\begin{proposition}
   For $p \in [1, \infty]$, $R_p^{\text{adv}}(\mhat{\beta})$ is convex in $\beta$.
\end{proposition}

\begin{proof}
  Let $\gamma \in [0, 1]$,  for $q \in [1, \infty]$, $\|\mhat{\beta}\|_q$ is a norm and from the triangular inequality, we have that:
  \begin{equation}
    \|\gamma \mhat{\beta}_1 + (1 -\gamma)\mhat{\beta}_1\|_q , \le \gamma \|\mhat{\beta}_1\|_q + (1 -\gamma)\|\mhat{\beta}_2\|_q. 
  \end{equation}
  Moreover,
  \begin{equation*}
    |y_0 - x_0^\trnsp(\gamma \mhat{\beta}_1 + (1 -\gamma)\mhat{\beta}_1 )| \le \gamma |y_0 - x_0^\trnsp\mhat{\beta}_1| + (1 -\gamma)|y_0 - x_0^\trnsp\mhat{\beta}_2|. 
  \end{equation*}
Hence, $h(\mhat{\beta}) = |y_0 - x_0^\trnsp\mhat{\beta}| +\|\mhat{\beta}\|_q$  is convex and, also, $h(\mhat{\beta}) \ge 0$ for all $\mhat{\beta}$. Now,  since $g(x) = x^2$ is convex and non-decreasing for $x\ge 0$, the composition $g \circ h$ is convex -- See~\citet[Section 3.2.4]{boyd_convex_2004}; moreover, the expected value of a convex function is also convex~\cite[Section 3.2.1]{boyd_convex_2004} and it follows that the right-hand side of Eq.~\eqref{eq:closeform-adv-risk} is convex.
\end{proof}

The results obtained for the adversarial risk are also valid for the empirical adversarial risk. Hence, it follows from Lemma~\ref{thm:advrisk-closeform} that:
  \begin{equation}
    \label{thm:advtraining-closeform}
    \mhat{R}_p^{\text{adv}}(\mhat{\beta}) =   \frac{1}{n}\sum_{i=1}^n\left(|y_i - x_i^\trnsp\mhat{\beta}| + \delta\|\mhat{\beta}\|_q\right)^2,
  \end{equation}
  and that it is convex. The above expression can be entered into a standard convex modeling language to obtain the adversarial training solution. In the numerical examples that follows we use CVXPY~\citep{diamond_cvxpy_2016} to train the model.

\subsection{Overparameterized models: Latent space feature model}

In this example, we consider artificially generated data from the latent space feature model described in Section~\ref{sec:adversarial-robustness}. The same experiment for the isotropic feature model is provided in the Supplementary Material, Section~\ref{sec:regularization-isotropic-model}. In Section~\ref{sec:adversarial-robustness}, we saw the unfortunate effect that if the input variables scale with $\eta(m) = \sqrt{m}$ (which corresponds to keeping $\Exp{\|x\|_2}$ constant as we vary the number of features $m$) we observe that $\|\mhat{\beta}\|_1$ grows indefinitely with $m$ when  $\mhat{\beta}$ was estimated using the minimum-norm solution. We also showed how this makes the  $\ell_\infty$-adversarial risk grow indefinitely with the number of features (i.e., Fig.~\ref{fig:three-types-of-risk}).

Let us now investigate if the same effect can be observed for models trained with ridge regression, lasso and adversarial training.
In Fig.~\ref{fig:norm-advtraining} we show the norm $\|\mhat{\beta}\|_1$ in these cases. For ridge regression the parameter norm grows with $\bigO(m)$ regardless of how large the regularization parameter $\delta$ is. We notice that $\ell_2$-adversarial training has a similar behavior for $\delta$ smaller then a certain threshold, in these cases it displays curves similar to ridge regression that grow with $\bigO(m)$. However, for sufficiently large values of the regularization parameter, the parameter norm of the solution is zero for all values of $m$.

For lasso, we see that the parameter norm goes to zero for overparameterized models with sufficiently large $m$. Looking at lasso as a bi-objective optimization problem helps interpret this behavior: as the number of features $m$ increases, the scaling affects the two objectives differently and the objective of keeping $\|\mhat{\beta}\|_1$ starts to be prioritized over the objective of keeping the square training error low, the more $m$ is increased. Interestingly, the $\ell_\infty$-adversarial training seems to behave in a very similar way.

In Fig.~\ref{fig:testerror-advtraining} we provide the adversarial test error for models trained with ridge regression, lasso and adversarial training, respectively.
As expected by our analysis of $\|\mhat{\beta}\|_1$, lasso and $\ell_\infty$-adversarial training yield solutions that do not deteriorate indefinitely. We believe this observation adds to our discussion about the role of scaling.
It highlights the fact that, even in the case of a mismatch between disturbance and how the input scales with the number of variables (i.e., $\E{x}{\|x\|_2^2}$  constant while we evaluate it under an $\ell_\infty$-adversary) it is still possible to avoid brittleness by considering a type of regularization that acts under the right norm.

\begin{figure}
    \centering
    \subfloat[Ridge regression]{\includegraphics[width=0.42\textwidth]{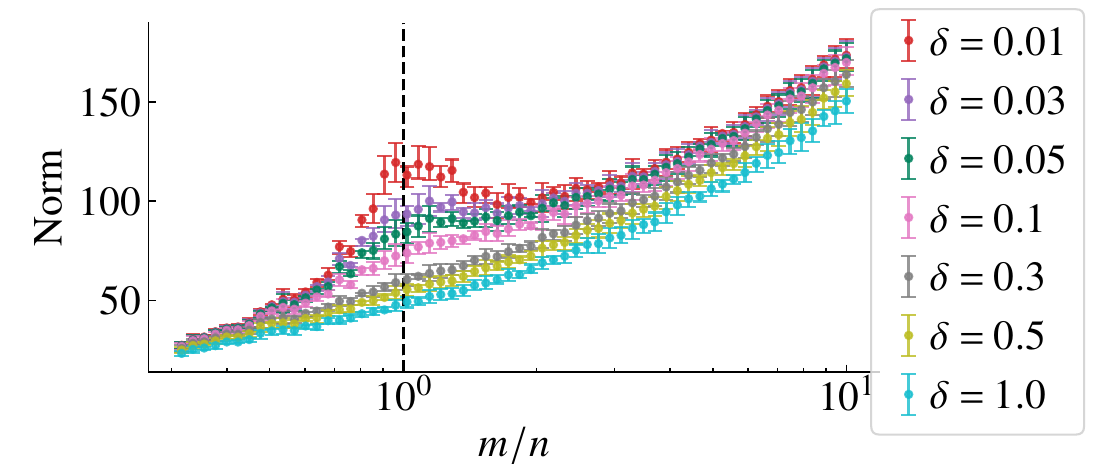}}\\
    \subfloat[Adversarial training $\ell_2$ ]{\includegraphics[width=0.42\textwidth]{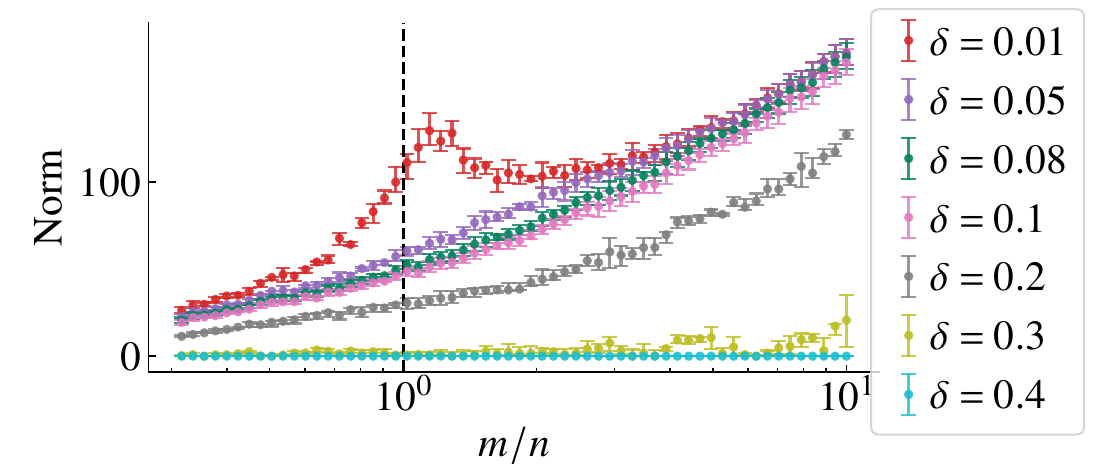}}\\
    \subfloat[Lasso regression]{\includegraphics[width=0.42\textwidth]{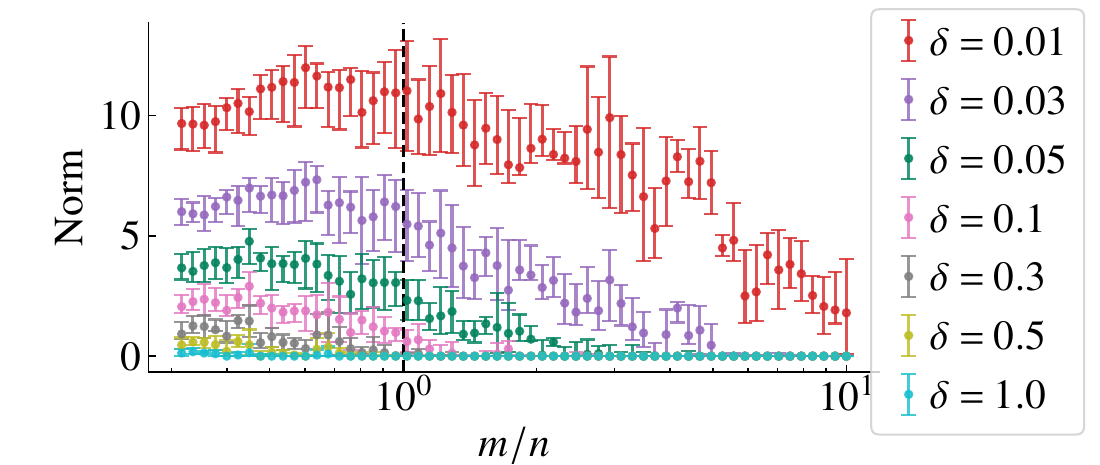}}\\
    \subfloat[Adversarial training  $\ell_\infty$  ]{\includegraphics[width=0.42\textwidth]{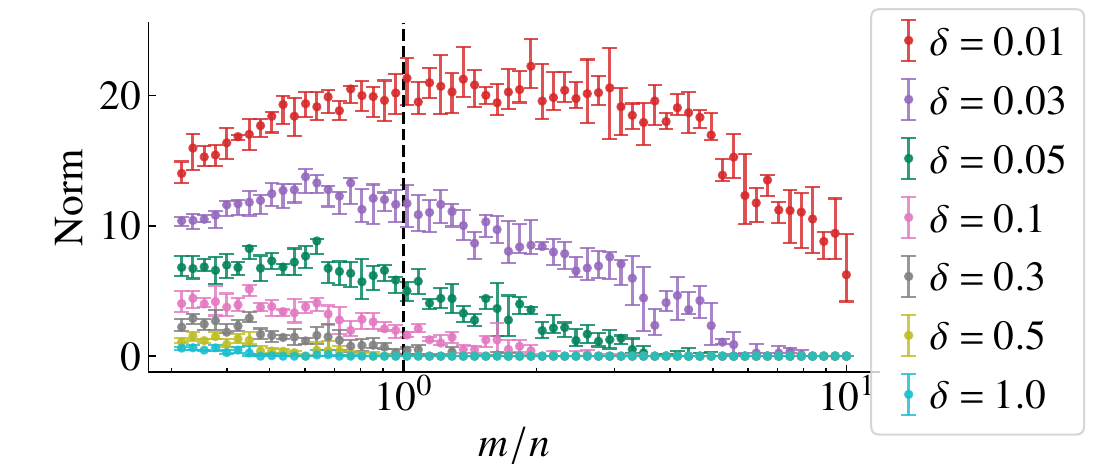}}\\
    \caption{\emph{Parameter norm $\|\mhat{\beta}\|_1$.} The input variables are scaled with $\eta(m) = \sqrt{m}$. The error bars give the median and the 0.25 and 0.75 quantiles obtained from numerical experiments (6 realizations) for a fixed training dataset of size $n=100$. We repeat the experiment for different amounts of regularization. The regularization parameter $\delta$ defined in Eq.~\eqref{eq:empirical_advrisk} for the adversarial training and in Eq.~\eqref{eq:ridge} and~\eqref{eq:lasso} for ridge regression and lasso.}
    \label{fig:norm-advtraining}
  \end{figure}

\begin{figure}
    \centering
    \includegraphics[width=0.5\textwidth]{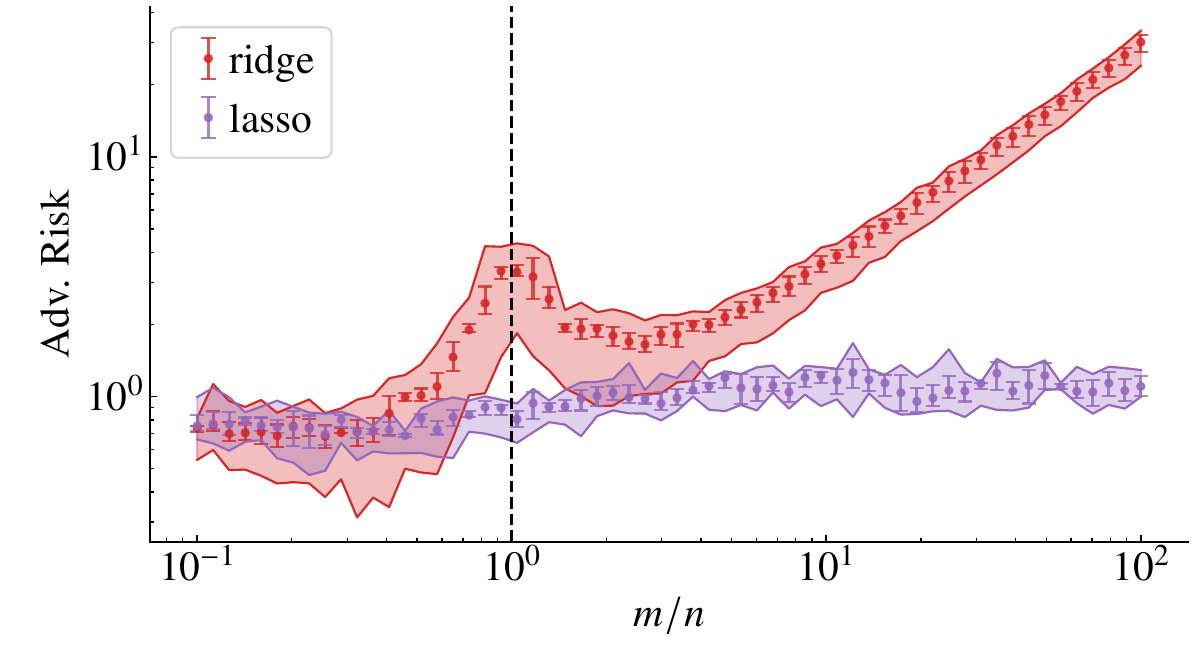}
    \includegraphics[width=0.5\textwidth]{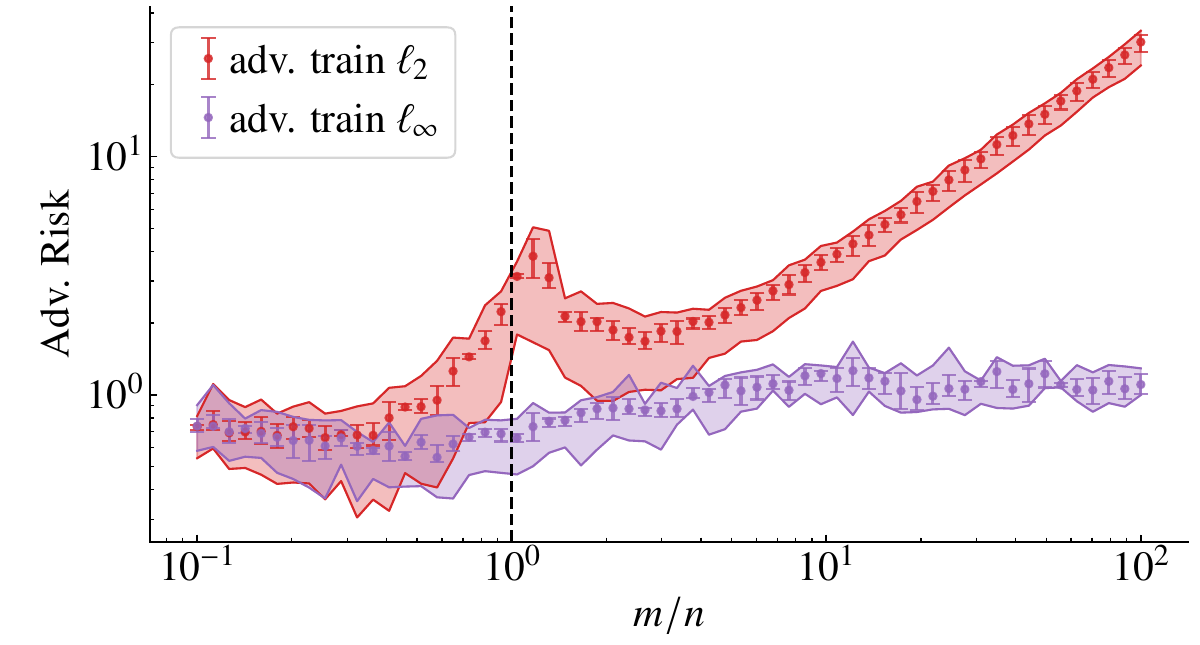}
 \caption{\emph{Adversarial $\ell_\infty$-risk.} \emph{Top:} ridge and lasso.  \emph{Bottom:}  $\ell_2$ and  $\ell_\infty$  adversarial training. On the $y$-axis we show the $\ell_\infty$-adversarial risk for models obtained by different training methods.   On the $x$-axis we have the ratio between  the number of features $m$ and the number of training datapoints $n$. The error bars give the median and the 0.25 and 0.75 quantiles obtained from numerical experiments (6 realizations). We use $\delta=0.01$ both during inference (to compute the adversarial risk) and during the adversarial training, as in Eq.~\eqref{eq:empirical_advrisk}. We also use $\delta=0.01$ for lasso and ridge regression, see Eq.~\eqref{eq:ridge} and~\eqref{eq:lasso}. Here the shaded region indicates the upper and lower bounds obtained empirically from Eq.~\eqref{eq:ineq_adversarial_risk}. The empirical risk in the test and the parameter norm are obtained from the experiments. }
    \label{fig:testerror-advtraining}
  \end{figure}

\section{Discussion}

\subsection{Related work}

\paragraph{Adversarial attacks}
The study of adversarial attacks pre-dates the widespread use of deep neural networks~\citep{dalvi_adversarial_2004, globerson_nightmare_2006}. An overview of earlier work is provided by \citet{biggio_wild_2018}. Nonetheless, the susceptibility of high-performance neural networks to adversarial attacks gave this framework higher visibility~\cite{bruna_intriguing_2014}. The framework of adversarial attacks has generated striking examples of the vulnerability of such models to very small input perturbations. Small changes in the input can cause a substantial drop in performance in otherwise state-of-the-art models, see for instance~\citep{bruna_intriguing_2014,goodfellow_explaining_2015, kurakin_adversarial_2018, fawzi_analysis_2018, ilyas_adversarial_2019, yuan_adversarial_2019}.

\paragraph{Robustness and the role of high-dimensionality} 
The conflict between robustness and high-performance models is explored by~\citet{tsipras_robustness_2019} and \citet{ilyas_adversarial_2019}. Indeed, one of the examples we give for the worst-case scenario of the $\ell_\infty$-adversarial error is motivated by an idea presented in~\citep{tsipras_robustness_2019}.  Moreover, simple examples where high-dimensional inputs yield easy-to-construct adversarial examples are abundant in the literature~\citep{gilmer_adversarial_2018}. 
The analysis of the robustness of more general nonlinear models, such as neural networks is provided in~\cite{bubeck_law_2021} and extended in \citet{bubeck_universal_2021}. They show how overparametrization can be a recipe for robustness. An alternative view is provided in~\cite{daniely_most_2020, bartlett_adversarial_2021}, where it is shown that ReLU neural networks can be made vulnerable since what these models learn is locally very similar to random linear functions. As we mentioned in the introduction, our work tries to reconcile these somewhat conflicting views in the context of linear models.

\paragraph{Double-descent} The double-descent performance curve has been experimentally observed for a variety of machine learning models, such as random Fourier features, random forests, shallow networks, transformers, convolutional networks and nonlinear ARX models; and for datasets obtained in diverse contexts, including image classification, natural language processing datasets and the identification of nonlinear dynamical systems~\citep{belkin_reconciling_2019, nakkiran_deep_2020, geiger_jamming_2019, geiger_scaling_2020, ribeiro_occam_2021b}. For instance, we illustrate the double-descent phenomena using random Fourier features in the Supplementary Material Fig.~\ref{fig:dd-diabetes}. Theoretical models for such phenomena are also often pursued:~\citet{bartlett_benign_2020} derive non-asymptotic bounds for linear regression models using concentration inequalities. In~\cite{geiger_jamming_2019} the authors draw connections with the physical phenomena of "jamming" in a class of glassy systems.~\citet{deng_model_2020} characterize logistic regression test error using the Gaussian min-max theorem.~\citet{muthukumar_harmless_2020} provides bounds on the risk.

Random matrix theory has been a useful tool for studying statistical phenomena. The framework and its potential for explaining and studying neural networks have been the focus of recent work \citep{pennington_nonlinear_2017, pennington_emergence_2018, pastur_random_2020}. It has also been a powerful tool in producing theoretical models for the double-descent phenomenon~\citep{belkin_two_2020, mei_generalization_2022, hastie_surprises_2019, advani_highdimensional_2020, adlam_neural_2020}. In our study of linear regression with random covariates, we make direct use of the asymptotic results obtained by~\citet{hastie_surprises_2019}. 

\paragraph{Analysis of adversarial attacks} Theoretical analysis of models under adversarial attacks is currently a rather popular topic. \citet{hassani_curse_2022} obtain exact asymptotics for random feature regressions. Also in the context of random feature regression and \citet{damour_underspecification_2020} provide asymptotics based on \citet{mei_generalization_2022}. 
They study a scenario where the adversarial attack is constrained to not change the risk.
The theoretical model is used to explain how underspecification might present a challenge in deployment and is backed by experiments. \citet{diochnos_adversarial_2018} consider the adversarial risk when the instances are uniformly distributed over $\{0, 1\}^n$. \citet{dohmatob_generalized_2019} provides a no-free-lunch theorem where it is shown that any classifier can be adversarially fooled with high probability when the perturbations are slightly greater than the natural noise level in the problem.

\paragraph{Adversarial attacks in linear models}
While a lot of current research focuses on adversarial examples for deep learning models, there is a growing body of work that study the fundamental properties of adversarial attacks in linear models.
There is a sound reason for this focus: linear models allow for analytical analysis while still reproducing phenomena of interest.

\citet{bhagoji_lower_2019} obtain optimal transport-based lower bounds for adversarial examples in classification problems. They consider Gaussian data and norm-bounded adversaries.
\citet{taheri_asymptotic_2021} derived asymptotics for adversarial training in binary classification.
Moreover, 
\citet{javanmard_precise_2020} provide asymptotics for adversarial attacks in linear regression. 
\citet{javanmard_precise_2020a} study classification settings. 

These asymptotics are often used to gain insight into the effect of adversarial training and adversarial robustness.
\citet{javanmard_precise_2020} studies the trade-off between adversarial risk and standard risk. Note that \citet{javanmard_precise_2020a} studies how overparametrization affects robustness to perturbations in the input and 
\citet{min_curious_2021} studies how the size of the dataset affects adversarial performance. We corroborate their observation that the adversarial performance might degrade as the size of the dataset increases.

The derivation of exact asymptotics is an impressive technical development, but we point out that it is not always trivial to gain insight from these results. 
The asymptotics obtained often do not have closed-form expressions and require the solution of either polynomials or integral equations.
Here, we advocate a simpler approach: approximating the adversarial risk using terms that often appear in other contexts. We believe that this is a powerful tool to gain insight into the problem, providing extra flexibility for quickly navigating between different setups. We use~\eqref{eq:ineq_adversarial_risk} and reduce the analysis to the risk and the parameter norm. Lemma~\ref{thm:advrisk-closeform} is an important tool for this analysis. 
We point out that \citet{xing_generalization_2021} proved a version of Lemma~\ref{thm:advrisk-closeform} specialized to the Gaussian case and $\ell_2$-norm and that~\citet{javanmard_precise_2020} state a version of the same lemma for the $\ell_2$-norm.

\paragraph{Rademacher complexity analysis}  Close to our work is that of \citet{yin_rademacher_2019}, which provided an analysis of $\ell_\infty$-adversarial attack on linear classifiers based on the Rademacher complexity. Their Theorem 1 resembles Equation~\eqref{eq:ineq_adversarial_risk}: we show that the adversarial risk and $(\text{risk} + \delta^2\|\mhat{\beta}\|_q^2)$ are upper and lower bounded by constant factors. \citet{yin_rademacher_2019} prove a similar relation for the adversarial Rademacher complexity of a linear classifier. In their proof they use a reformulation of the adversarial loss similar to that of Lemma~1 but for classifiers. Similar to our results, they showed an unavoidable dimension dependence unless the weight vector has a bounded norm. On the one hand, our work extends their results for regression. On the other hand, by analysing general $\ell_p$-adversarial attacks and different covariate scalings we studied a wide variety of possible behaviors that they do not observe by focusing only on $\ell_\infty$-adversarial attacks and training.

\subsection{Connections to neural networks}

The success of deep neural networks is an important reason for digging deeper into the properties of overparameterized models. Here, however, we study the phenomenon in linear models. To motivate the relevance of our study also for neural networks, we appeal to a recent line of work that has pointed out a direct connection between linear models and more complex models such as neural networks~\cite{allen-zhu_convergence_2019,chizat_global_2018,du_gradient_2019,jacot_neural_2018}. The idea can be understood in simple terms. Let  the parameterized function $f(\cdot; \theta)$ denote the neural network, where $\theta \in \mathbb{R}^m$ denote the vector of parameters. Assume that the number of parameters is very large and that training the neural network moves each of them just by a small amount w.r.t. its initialization $\theta_0$.  A linearization of the model around $\theta_0$, yields 
\begin{equation}
\label{eq:nn_linearization}
f(x; \theta) \approx  f(x; \theta_0) + \nabla_\theta f(x; \theta_0)^\top\tilde{\theta},
\end{equation}
where $\tilde\theta = \theta - \theta_0$. Hence, the problem can be approximated by an affine problem that could be solved using linear regression. Indeed, it can be established that as the neural network becomes infinitely wide  the training of the neural network actually becomes solving a problem similar to that in \eqref{eq:nn_linearization}.

Linear models are also a natural setup to study adversarial attacks.
Indeed, while there was initial speculation that the highly nonlinear nature of deep neural networks was the cause of its vulnerabilities to adversarial attacks~\citep{bruna_intriguing_2014}, that the idea was later dismissed and the vulnerabilities can be observed already in purely linear settings~\citep{goodfellow_explaining_2015}. 

\subsection{Extension to nonlinear models}
A question that naturally comes to mind is if parts of this analysis can 
be generalized to nonlinear settings. For that, define the adversarial risk associated with a given function $f$ by:
\begin{equation}
    R^{\text{adv}}_p(f)= \E{x_0, y_0}{\max_{\|\Delta x_0\|_p \le \delta}(y_0 - f(x_0 + \Delta x_0))^2}.
  \end{equation}
Let $L(f)$ be the Lipschitz constant of $f$, i.e. $|f(x_1) - f(x_2)|\le L(f) \|x_1 - x_2\|$ for all $x_1, x_2$. The idea of using the Lipschitz constant as a proxy for robustness is quite standard and a common procedure to obtain robust models is to jointly optimize the risk $R(f)$ and the Lipschitz constant $L(f)$, see e.g.~\cite{fazlyab_efficient_2019}. Indeed, an analysis equivalent to the one used in the proof of Lemma~\ref{thm:advrisk-closeform} yields
\begin{equation}
R^{\text{adv}}_p(f) \le \E{x_0, y_0}{ \left(|y_0 - f(x_0)| + \delta L(f)\right)^2}.
\end{equation}
In this case, equality does not necessarily hold. Proposition~\ref{thm:tight-holder} was used in the proof for the linear case, but there is not an obvious equivalent in the nonlinear case. Hence, instead of the approximation~(\ref{eq:ineq_adversarial_risk}) we would only have an upper bound (and no lower bound).

\section{Conclusion}
\label{sec:discussion}

In this paper, we focus on the behavior of the adversarial risk as we change the number of features. Our analysis is based on the fact that $\ell_p$-adversarial risk is between 1 and 2 times $(\text{risk} + \delta^2\|\mhat{\beta}\|_q^2)$, where $q$ is the complementary norm to $p$. Hence, the behavior of the adversarial risk can be studied by analysing these two components. We use such results to analyse the role of high-dimensionality in the performance of linear models under adversarial attacks.

On the one hand, the result implies that $\ell_2$-adversarial risk presents a double-descent curve when both the risk and $\|\mhat{\beta}\|_2$ present such behavior. We use asymptotic results from~\citet{hastie_surprises_2019} to illustrate a double-descent curve in the adversarial risk for models with features randomly generated with isotropic, equicorrelated and spiked (i.e., latent space model) covariance matrices.

On the other hand, we focus on the analyse of cases where the risk is small but the $\ell_p$-adversarial risk grows with the number of features. In our setup, as a direct consequence of the aforementioned approximation of the adversarial risk, this happens if and only if $\delta\|\mhat{\beta}\|_q\rightarrow \infty$ as the number of features  $m \rightarrow \infty$. In order to analyse the term $\delta\|\mhat{\beta}\|_q$,
\begin{itemize}
\item we use non-asymptotic analysis for the norm of the estimated parameter obtained by the minimum-norm solution. For Gaussian covariates, we show that: $\|\mhat{\beta}\|_2 = \bigO(1/m)$, while $\|\mhat{\beta}\|_1= \bigO(1)$.
\item we show that for isotropic covariates, if $\delta \propto E\{\|x\|_2\}$ then  $\delta = \bigO(m)$. Furthermore, for sub-Gaussian covariates we have that if $\delta \propto E\{\|x\|_\infty\}$, then $\delta = \bigO\left(\sqrt{\log(m}\right)$.
\end{itemize}
We combine the two results to show examples that are robust to $\ell_2$-adversarial attacks but can be made vulnerable to $\ell_\infty$-adversarial attacks as we increase the number of features. The most pathological results are usually obtained in a mismatched situation, where we apply an $\ell_\infty$-adversarial attack with magnitude $\delta \propto E\{\|x\|_2\}$. In this case, we have shown that the adversarial risk can be made arbitrarily large (i.e., the model is arbitrarily vulnerable to an adversary) as the number of features grows. Such a mismatched setup (with a $\ell_\infty$-adversarial attack with scaling proportional to the $\ell_2$-norm) is present in influential examples such as those in~\citep{tsipras_robustness_2019, goodfellow_explaining_2015}, and the mismatch often appears hidden in the argument. Finally, we also provided a convex optimization formulation of adversarial training and studied similarities between adversarial training and parameter-shrinking methods.

\section*{Acknowledgement}
The authors would like to thank Dave Zachariah for very fruitful discussions. 
This research was financially supported by the project \emph{Deep probabilistic regression -- new models and learning algorithms } (contract number: 2021-04301), funded by the Swedish Research Council and by \emph{Kjell och M{\"a}rta Beijer Foundation}.

\printbibliography

\newpage~
\newpage

\appendices
\pagenumbering{roman} 

\setcounter{page}{1}

\setcounter{equation}{0}
\renewcommand{\theequation}{S.\arabic{equation}}%

\setcounter{figure}{0}
\renewcommand{\thefigure}{S.\arabic{figure}}%

\section{Additional Proofs}
\label{sec:additional-proofs}

\subsection{Proof Proposition 2}
Let us now prove the first statement: \textit{if $1/p + 1/q = 1$ and $\Delta x_i = \text{sign}(\beta_i) |\beta_i|^{q/p}$ than $|\beta^\trnsp\Delta x| = \|\beta\|_q \|\Delta x\|_p$.} On the one hand,
\begin{align*}
\|\Delta x\|_p = \left(\sum |\Delta x_i |^p\right )^{1/p} = \left(\sum |\beta_i|^{q} \right)^{1/p}  = \|\beta_i\|^{q/p}_q.
\end{align*}
Hence,
\begin{align*}
\|\Delta x\|_p\|\beta_i\|_q=  \|\beta_i\|^{q}_q = \sum_{i =1}^m |\beta_i|^{q},
\end{align*}
where $q/p+1 = q$ was used in the first equality. Using this again we obtain that $\|\Delta x\|_p\|\beta_i\|=\sum_{i =1}^m |\beta_i|^{q/p} |\beta_i|$ which is equal to $|\Delta x_i^\trnsp \beta| $ by the definition of $\Delta x_i$. The other two statements can be verified in a similar way by replacing the given $\Delta x$ into the formula and verifying the equality.

\subsection{Inequalities~\eqref{eq:ineq_adversarial_risk} and~\eqref{eq:bounds-linear-regression}}

The two inequalities~\eqref{eq:ineq_adversarial_risk} and~\eqref{eq:bounds-linear-regression} follow from~\eqref{eq:closeform-adv-risk1}. The first can be derived  by a direct application of Jensen's  inequality that yields $0 \le\E{x_0, y_0}{|e_0|} \le \sqrt{ R(\mhat{\beta})}$. The second inequality follows from a direct application of the Cauchy-Schwartz inequality: $0 \le\E{y, x_0, y_0}{\norm{\mhat{\beta}}_q |e_0|} \le \sqrt{L_q R}$.

\subsection{Proof of Lemma~\ref{thm:bias-variance-decomposition}}

 \textbf{Proof for $ \|\mhat{\beta}\|_2$:}
From Eq.~\eqref{eq:linear-data-model} and Eq.~\eqref{eq:min-norm-sol} it follows that:
\begin{equation}
    \label{eq:betahat-close-formula}
    \mhat{\beta} = \underbrace{(X^\trnsp X)^\pinv X^\trnsp X}_{\Phi} \beta + \underbrace{(X^\trnsp X)^\pinv}_{\frac{1}{n}\hat{\Sigma}^\pinv} X^\trnsp \epsilon.
\end{equation}
\hspace{-2pt}
Hence, since $\hat{\Sigma}$ is symmetric:
\begin{equation}
\mhat{\beta}^\trnsp \mhat{\beta} = \beta^\trnsp \Phi^\trnsp \Phi \beta + \frac{1}{n} \beta^\trnsp \Phi^\trnsp \hat{\Sigma}^\pinv X^\trnsp \epsilon  +  \frac{1}{n^2} \epsilon^\trnsp X \hat{\Sigma}^\pinv \hat{\Sigma}^\pinv X^\trnsp \epsilon ,
\end{equation}
where the first term is equal to $\beta^\trnsp \Phi \beta$, since $\Phi$ is an \textit{orthogonal projector} i.e., $\Phi^\trnsp = \Phi$ and  $\Phi  \Phi =  \Phi$. Moreover, the middle term has zero expectation. 

Now, since the second term is a scalar, it is equal to its trace. Using the fact that the trace is invariant over cyclic permutations, 
$\epsilon^\trnsp X \hat{\Sigma}^\pinv \hat{\Sigma}^\pinv X^\trnsp \epsilon = \tr \brackets{\hat{\Sigma}^\pinv X^\trnsp \epsilon\epsilon^\trnsp X \hat{\Sigma}^\pinv}.$

From the assumption that the noise samples are independent and have variance $\sigma^2$, we have $\E{\epsilon}{\epsilon \epsilon^\trnsp} = \sigma^2 I$, where $I$ is the identity matrix. Since we can swap the trace and the expectation operator we obtain
\begin{equation*}
\E{\epsilon}{\mhat{\beta}^\trnsp \mhat{\beta}} = \beta^\trnsp \Phi \beta +  \frac{1}{n^2}\tr \brackets{\hat{\Sigma}^\pinv X^\trnsp \underbrace{\E{\epsilon}{{\epsilon \epsilon^\trnsp}}}_{\sigma^2 I} X \hat{\Sigma}^\pinv }.
\end{equation*}
The results follow from the definition of $\hat{\Sigma}^\pinv$ and the following property of pseudo-inverse 
$\hat{\Sigma}^\pinv\hat{\Sigma}\hat{\Sigma}^\pinv = \hat{\Sigma}^\pinv$.

\textbf{Proof for $R$:}
Now, 
\begin{equation}
    \label{eq:matrix-equationing-risk}
    R(\mhat{\beta}) =\E{x_0, y_0}{(\beta^\trnsp x_0 - y_0)^2} = (\beta-\mhat{\beta})^\trnsp \Sigma (\beta-\mhat{\beta}) + \sigma^2.
\end{equation}

From~\eqref{eq:betahat-close-formula} it follows that
\begin{equation*}
    \beta - \mhat{\beta} = \underbrace{(I - \Phi)}_{\Pi} \beta + \frac{1}{n}\hat{\Sigma}^\pinv X^\trnsp \epsilon,
\end{equation*}
where $\Pi$ is again an orthogonal projector, i.e., $\Pi^\trnsp = \Pi$  and $\Pi\,\Pi =  \Pi$. We can then compute a closed-form expression for $\E{\epsilon}{R(\mhat{\beta})}$  using the same procedure as above.

\subsection{Proof of Proposition~\ref{thm:scaling-motiv}}
\label{sec:isotropic-features-norms}

Here $x \in \R^m$. To compute the expected $\ell_2$-norm of $x_i$ we use the fact that $\Cov{x} = I_m$. Hence, $\E{x}{\|x\|_2^2} = \E{x}{\tr\{x x^\trnsp\}} = \tr \{\E{x}{x x^\trnsp}\}  = m$. 
Let us proceed by analyzing the $\ell_\infty$-norms. Let $t>0$ be an arbitrary value to be chosen later. Via direct use of Jensen's inequality we obtain
\begin{multline}
    \exp (t\E{x}{\|x\|_\infty}) \le \Exp{\exp (t\|x\|_\infty)} = \\ \Exp{\max_j \exp (t |x^{j}|)} \le \sum_{j = 1}^m\left(\Exp{\exp (t |x^{j}|)}\right).
\end{multline}
Using the fact that $x^{j}$ is sub-Gaussian (without loss of generality, we assume unitary proxy variance), we have $\Exp{\exp (t|x^{j}|)}  \le \exp(t^2/ 2)$. Therefore, $\exp (t\Exp{\|x\|_\infty}) \le 2m \exp(t^2/ 2)$ and 
\begin{equation}
    \Exp{\|x\|_\infty} \le \frac{\log(2m)}{t} +  \frac{t}{2}.
\end{equation}
Here, we can choose $t = 2\sqrt{\log(2m)}$, which yields $\Exp{\|x\|_\infty} \le \sqrt{2\log(2m)}$.
We conclude that $\Exp{\|x\|_\infty} = \bigO(\sqrt{\log(m)})$. It is also possible to obtain $\Exp{\|x\|_\infty} = \Omega(\sqrt{\log(m)})$ using a similar argument to the one described in \textit{``Bounds on the Expectation of the Maximum of Samples from a
Gaussian''} by Gautam Kamath (\href{www.gautamkamath.com/writings/gaussian_max.pdf}{www.gautamkamath.com/writings/gaussian\_max.pdf} [online accessed: 2021-10-11]). Hence, $\Exp{\|x\|_\infty} = \Theta(\sqrt{\log(m)})$.
The result for is also provided in~\citet[Exercises 2.5.10 and 2.5.11]{vershynin_high-dimensional_2018}.

\subsection{Proof of Proposition}

We used the change of variables $\Delta x = \eta(m) \widetilde{\Delta x}$ which does not change the result of the maximization:
\begin{footnotesize}
\begin{align*}
\footnotesize
\text{adv-error}\left(x, \beta,  \eta(m)\delta  \right) &= \max_{\|\Delta x\|_2 \le \eta (m) \delta }(y - (x + \Delta x)^\trnsp\beta)^2 \\
 &= \max_{\|\widetilde{\Delta x}\|_2 \le \delta }\left(y - \left(x + \eta(m)\widetilde{\Delta x}\right)^\trnsp\beta\right)^2 \\
 &= \max_{\|\widetilde{\Delta x}\|_2 \le \delta }\left(y - \left(\frac{x}{\eta(m)} + \widetilde{\Delta x}\right)^\trnsp\left(\eta(m)\beta\right)\right)^2 \\
 &= \text{adv-error}\left(\frac{x}{\eta(m)}, \eta(m)\beta, \delta\right)
\end{align*}
\end{footnotesize}

\subsection{Proof of \eqref{thm:concentration-around-the-mean-norml1}}

Let us define the function $f(\Phi) = \|\Phi \beta\|_1$. The function $f:G(m, n) \rightarrow \R $ is Lipschitz with constant $\|\beta\|_2 \sqrt{m} $ since
$$
\begin{aligned}
|f(\Phi_1) - f(\Phi_2) | &\le \|(\Phi_1 - \Phi_2) \beta\|_1 \\
& \le \sqrt{m} \|(\Phi_1 - \Phi_2) \beta\|_2  \\
&\le  \sqrt{m} \|\beta\|_2  \|\Phi_1 - \Phi_2\|_2, 
\end{aligned}$$
where we first used the triangular inequality. Then Lemma~\ref{thm:relation-between-p-norm} provided an upper bound on the $\ell_1$-norm by the $\ell_2$-norm, and finally the definition of the operator norm of a matrix ${\|\Phi\|_2 = \sup_\beta \left( \|\Phi\beta\|_2 / \|\beta\|_2 \right)}$. The result in Eq.~\eqref{thm:concentration-around-the-mean-norml1} now follows from Theorem 5.2.9  in~\citet{vershynin_high-dimensional_2018}.

\section{Illustration of the double-descent phenomena}
\label{sec:dd}

  \begin{figure}[H]
    \centering
    \subfloat[Mean square error]{\includegraphics[width=0.8\columnwidth]{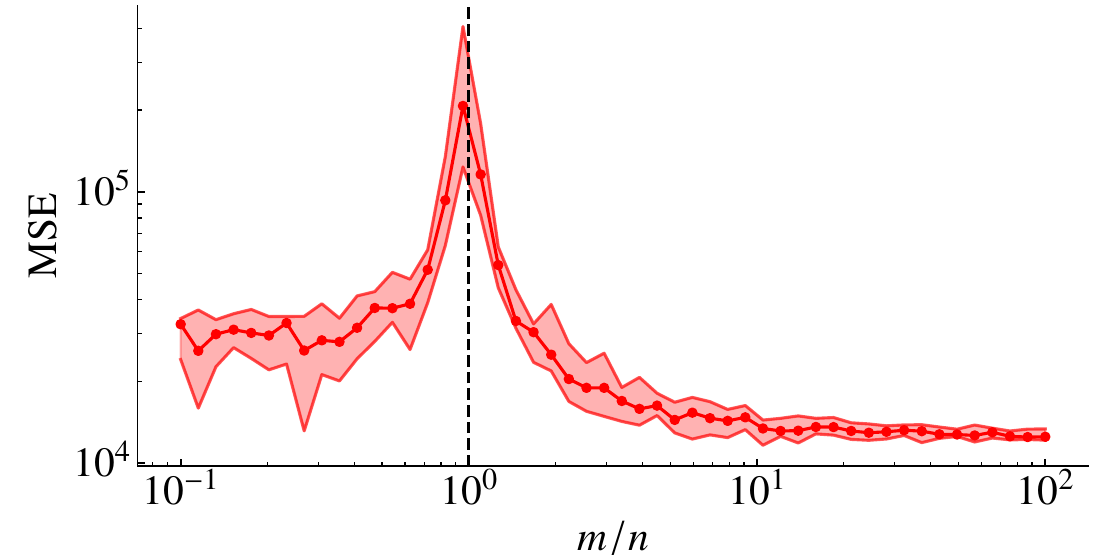}}\\
    \subfloat[Parameter $\ell_2$-norm]{\includegraphics[width=0.8\columnwidth]{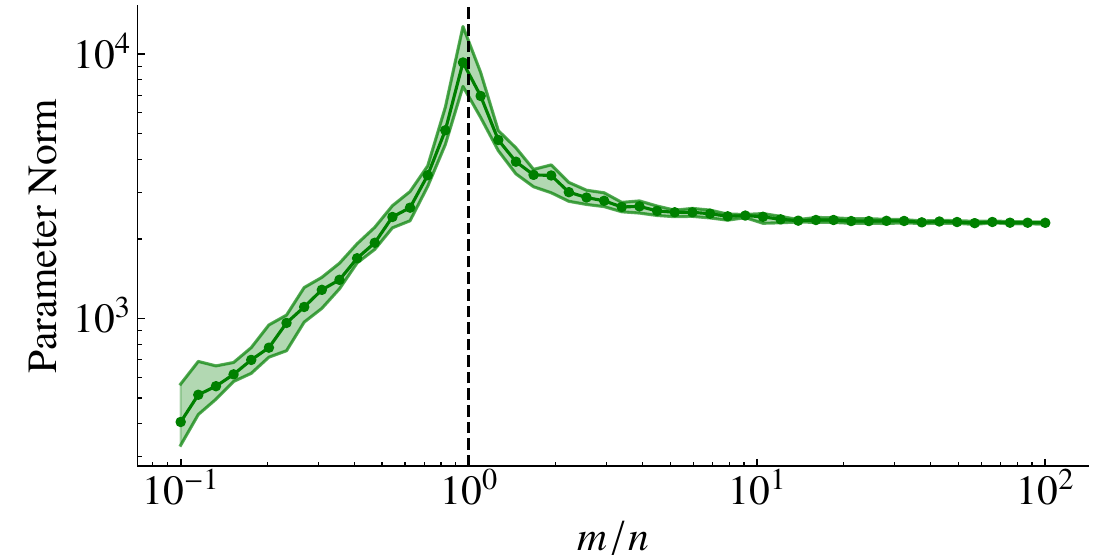}}
 \caption{\textit{Double-descent in diabetes dataset~\cite{efron_least_2004}.} Here we give a simple illustration of the double-descent phenomena using a real dataset. The dataset has $10$ baseline variables (age, sex, body mass index, average blood pressure, and six blood serum measurements), which were obtained for $442$ diabetes patients. Data from $n= 300$ patients was used as training and the remaining patients for testing. To this data we adjusted a nonlinear model that is linear-in-the-parameters $\hat{y}_i =  \sum_{i=1}^m \beta_i \phi_i(x)$. Here, $\phi_i$ are maps obtained by means of Random Fourier Features~\cite{rahimi_random_2008}. This model can be understood as a one-layer neural network where only the readout weights are trained. We use $m$ to denote the number of features. When $m > n$ and there are multiple solutions we choose the solution with minimum parameter norm. As a function of the ratio  between the number of features and the number of training data points $m/n$ we display; (a) the mean square error in the test dataset; and, (b) the parameter norm. The solid line is the median of the 10 experiments and the shaded region indicates the inter-quartile range.}
    \label{fig:dd-diabetes}
  \end{figure}

\section{Isotropic feature model}
\label{sec:isotropic-features-model-extended}

The asymptotic behavior of $R$ and $L_2$ for isotropic features is described in Lemma~\ref{thm:asymptotics-linear-regression}. In Fig.~\ref{fig:isotropic-predrisk-and-norm}(a), we illustrate the behavior of the risk for different values of $r^2$. Again, as in the main text, we refer to the ``null risk'' as the quantity $r^2 + \sigma^2$ that correspond to the risk of the null estimator $\tilde{\beta} = 0$.  In the underparameterized region, the prediction risk is smaller then the null risk iff $\frac{m}{n} < \frac{r^2}{\sigma^2 + r^2}$. In the overparameterized region, when $r^2 > \sigma^2$, as in the situation studied in the main text, the prediction risk has a local minima at $\gamma =  \frac{r}{r+ \sigma}$. Furthermore, it approaches the null risk from below as $\gamma \rightarrow \infty$. If $r^2 < \sigma^2$, the prediction risk decreases monotonically, approaching the null risk from above as $\gamma \rightarrow \infty$.
As mentioned in the main text, the prediction risk does not change as the inputs are rescaled. However, the parameter norm does change. The parameter norm is shown for different scaling in  Fig.~\ref{fig:isotropic-predrisk-and-norm}(b). We discuss the behavior for each case in the main text. We also illustrate the effect of different input scalings on the adversarial risk in Fig.~\ref{fig:double-descent-l2-isotropic}.

\begin{figure}
    \centering
    \subfloat[Prediction risk]{\includegraphics[width=0.5\columnwidth]{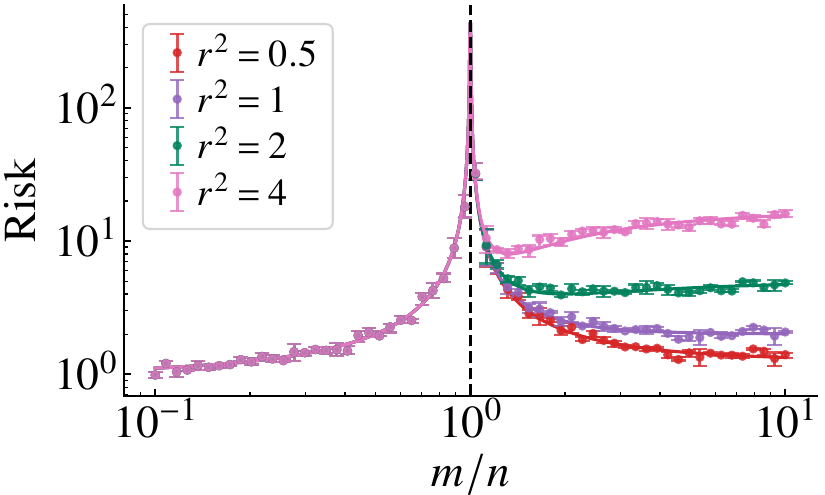}}
    \subfloat[$\ell_2$ parameter norm]{\includegraphics[width=0.5\columnwidth]{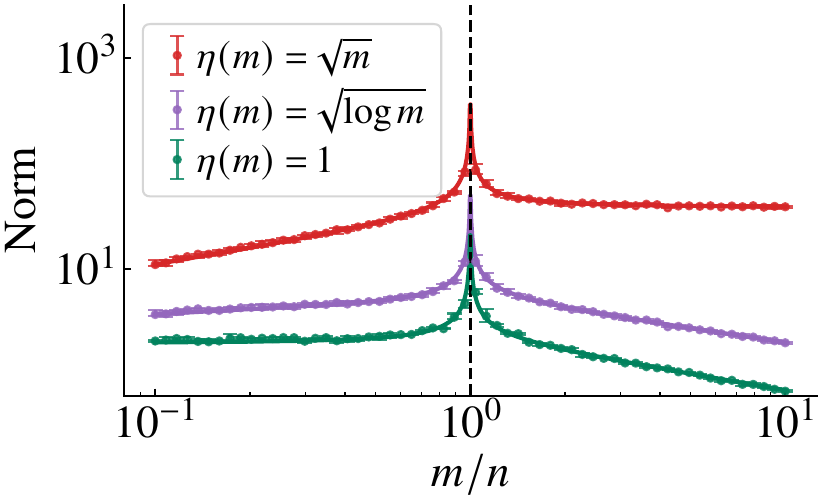}}
    \caption{\emph{Isotropic features.}  In (a), we show the \textbf{prediction risk} for $r^2 \in \{0.5, 1, 2, 4\}$. In (b), we show the the \textbf{parameter norm} for $r^2 = 2$ and for different choices of scaling of the input. The remaining parameters are the same as in Fig.~\ref{fig:double-descent-l2-isotropic}.}
    \label{fig:isotropic-predrisk-and-norm}
\end{figure}

\begin{figure}
    \centering
    \begin{tabular}{c|c|c|c}
        & $r^2 = 0.5$ &  $r^2 = 1$ &  $r^2 = 2$\\\hline
        \rotatebox{90}{\footnotesize $\,\,\,\,\,\,\,\,\,\eta(m) = 1\,\,\,\,\,\,\,\,\,\,$} & \includegraphics[width=0.13\textwidth]{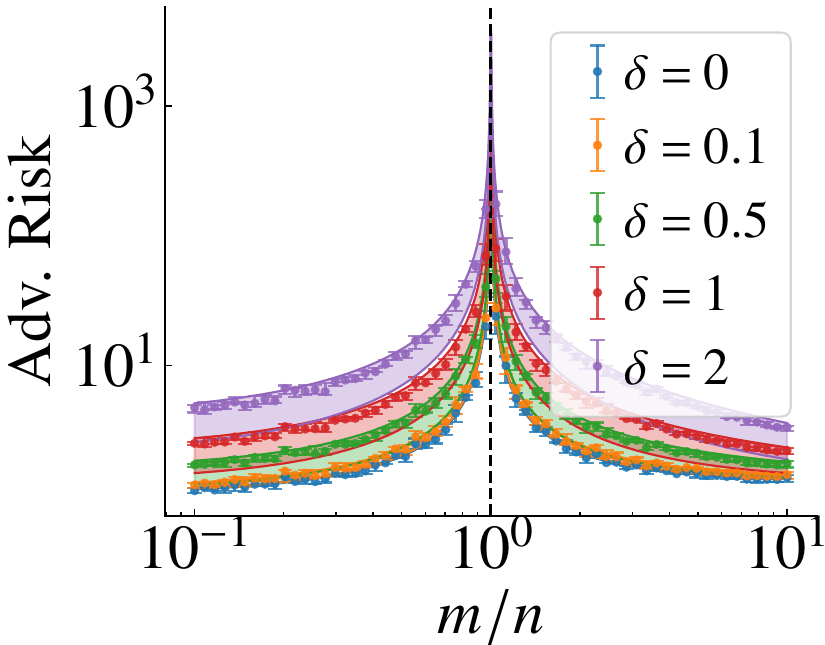} & \includegraphics[width=0.13\textwidth]{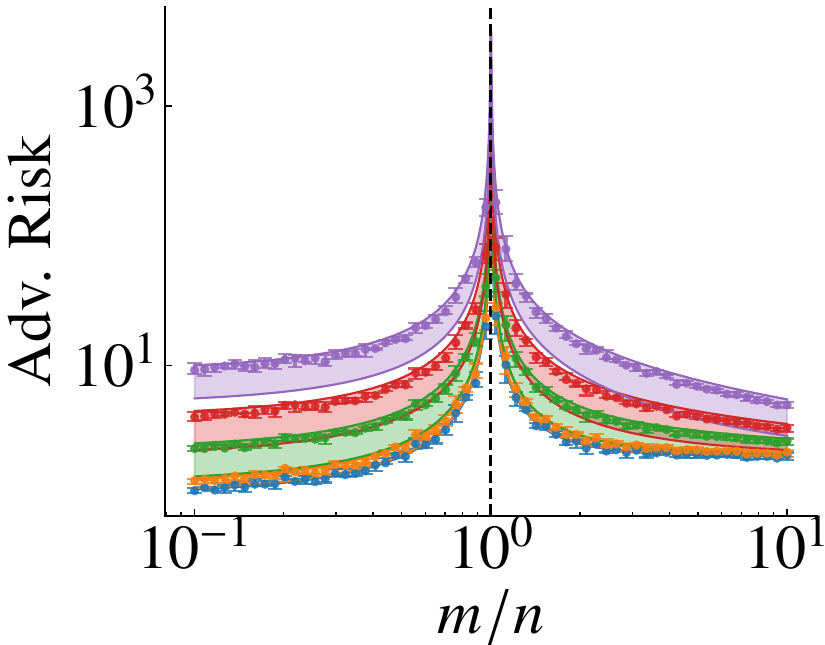}& \includegraphics[width=0.13\textwidth]{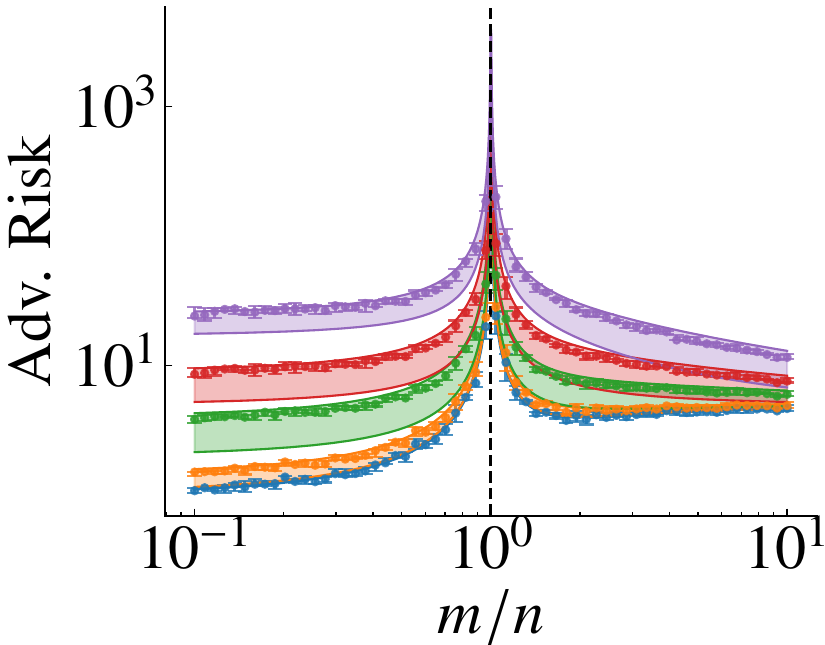}\\\hline
        \rotatebox{90}{\footnotesize $\,\,\eta(m) = \sqrt{\log{m}}\,\,\,\,$} & \includegraphics[width=0.13\textwidth]{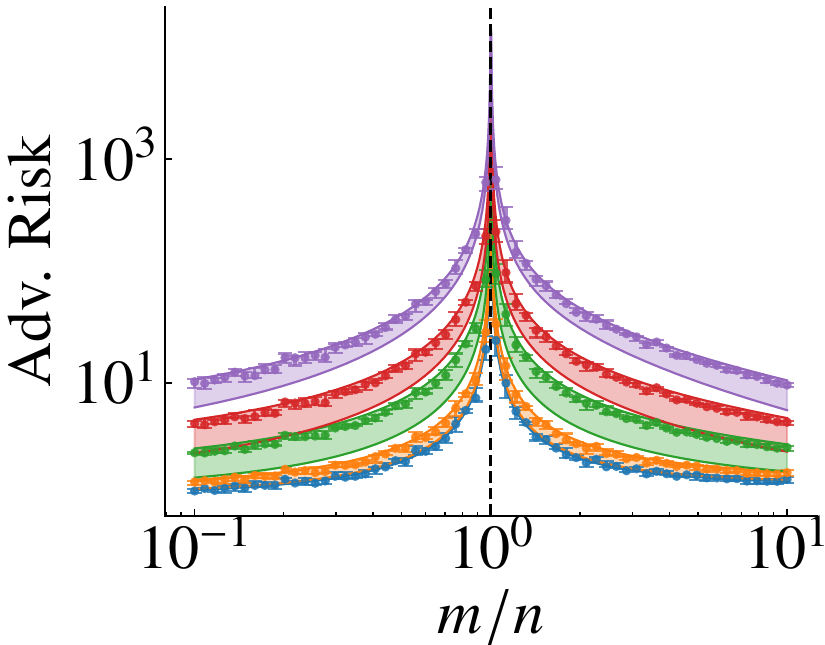} & \includegraphics[width=0.13\textwidth]{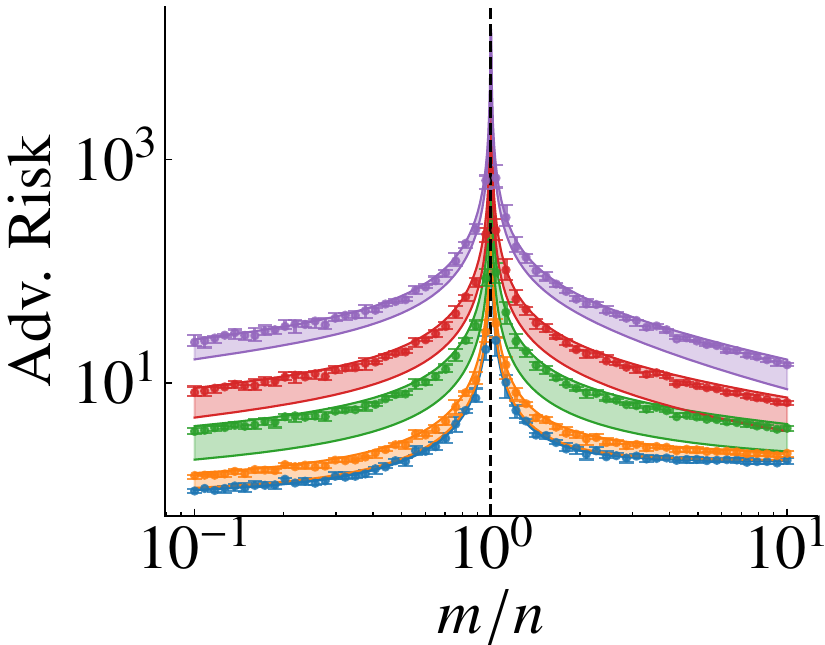}& \includegraphics[width=0.13\textwidth]{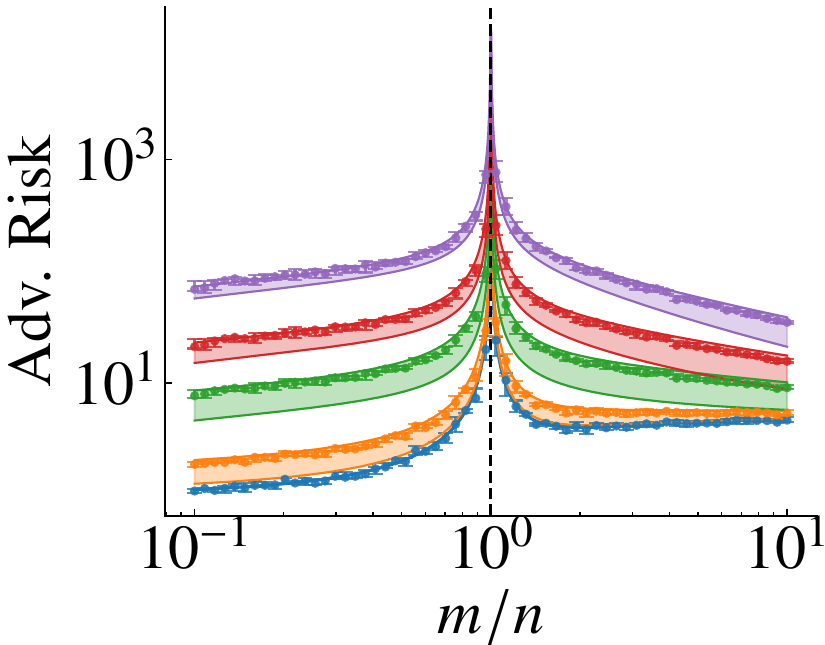} \\\hline
        \rotatebox{90}{\footnotesize $\,\,\,\,\,\,\,\, \eta(m) = \sqrt{m} \,\,\,\,\,\,\,\,$} & \includegraphics[width=0.13\textwidth]{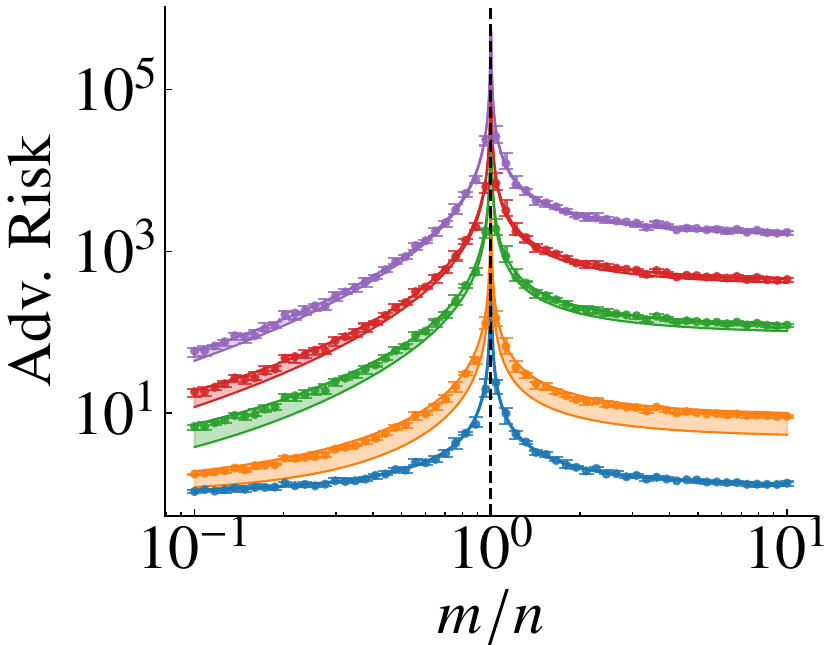} & \includegraphics[width=0.13\textwidth]{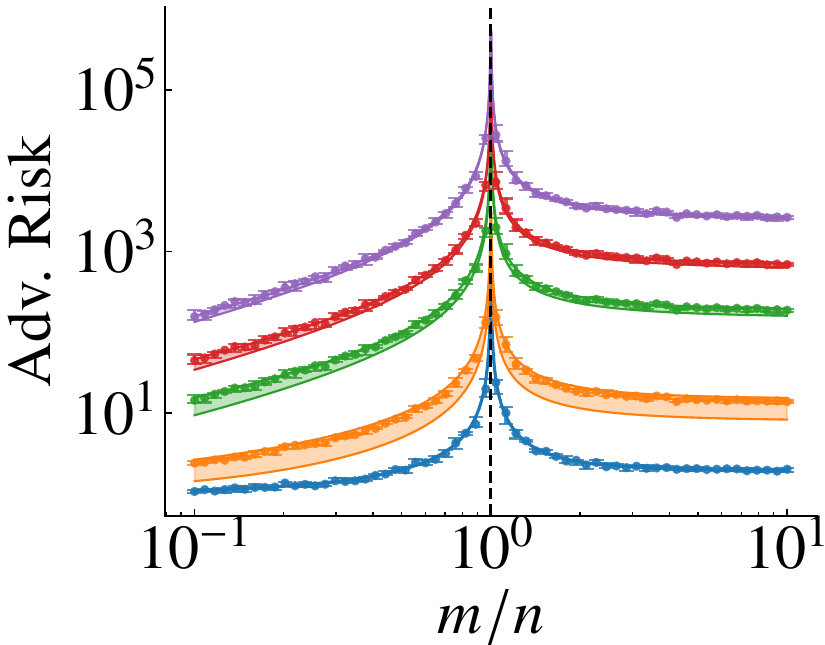}& \includegraphics[width=0.13\textwidth]{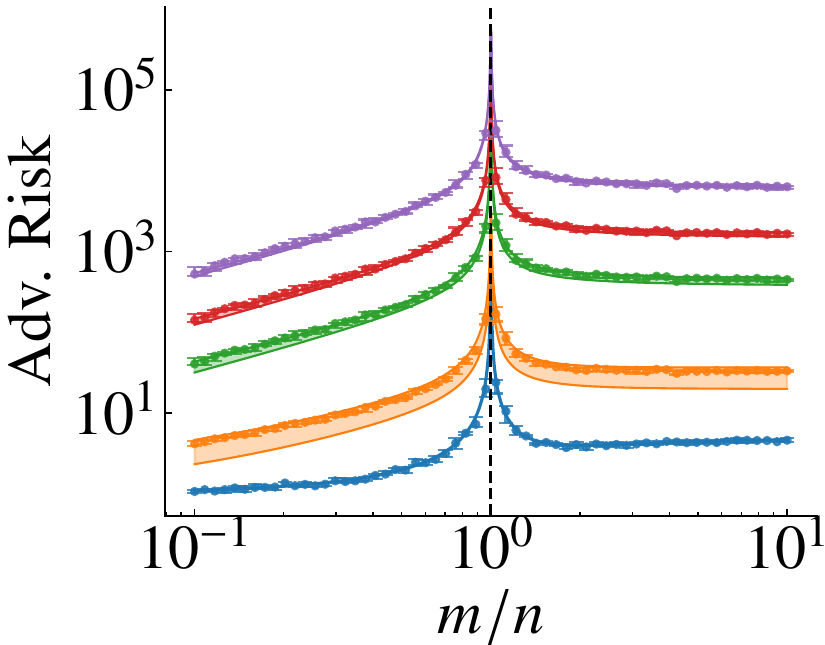}
    \end{tabular}
    \caption{\emph{Adversarial $\ell_2$-risk, isotropic features.}  The setup is the same as in Fig.~\ref{fig:double-descent-l2-isotropic} but for a \textit{different scaling} and different values of the parameter norm $r^2$.}
    \label{fig:double-descent-l2-isotropic-all}
\end{figure}

\section{Equicorrelated features model}
\label{sec:equicorrelated-model}

Here, we consider the case where the features are $\rho$-equicorrelated. Let $\Sigma$ be such that its $(i,j)$-th entry is $\Sigma_{i, j} = 1$ if $i = j$ and  $\Sigma_{i, j} = 
\rho$ otherwise
and $x = \Sigma^{1/2} z$ for $z$ composed of i.i.d. features with zero mean, unitary variance, and bounded moments of order greater than 4 that is finite. In this case, the following result holds:
\begin{lemma}
\label{thm:equicorrelated-features}
Assume that $x_i$ is generated as described above. Also, assume that $\beta\sim \N\left(0, \frac{r^2}{m}I\right)$. Then, as $p, n\rightarrow \infty$ $p/n \rightarrow \gamma$, it holds almost surely that:
\begin{eqnarray}
    \E{\beta}{R} &\rightarrow&
    \begin{cases}
    \sigma^2 \frac{\gamma}{1 - \gamma}, \gamma < 1, \\
    r^2 (1-\rho)(1 - \frac{1}{\gamma}) + \sigma^2 \frac{1}{\gamma - 1}, \gamma > 1.
    \end{cases}\\
     \E{\beta}{L_2}  &\rightarrow &
    \begin{cases}
    r^2 + \sigma^2 \frac{\gamma}{(1 - \gamma)(1-\rho)}, \gamma < 1, \\
    r^2 \frac{1}{\gamma} + \sigma^2 \frac{1}{(\gamma - 1)(1-\rho)}, \gamma > 1.
    \end{cases}
\end{eqnarray}
\end{lemma}

The asymptotics for $R$ are presented in Corollary 7 from \citet{hastie_surprises_2019}. The proof for the asymptotics of $L_2$ follows from Corollary 2 of the same paper and relies on the nice properties of the equicorrelated matrix. The next proposition gives the eigenvalues and eigenvectors of such a matrix. We use $\vec{\mathbf{1}}$ to denote a vector of dimension $m$ with all its entries equal to 1.
Furthermore, $s_i$, $i = 1, \dots, n$ denotes the eigenvalues of $\Sigma$ and $v_i$ the corresponding eigenvectors.
\begin{proposition}
    \label{thm:equicorrelated-eigdecomposition}
    Let $\Sigma\in\R^{m\times m}$ be an equicorrelated matrix. 
    Then $s_1 = 1 + (m-1) \rho$  and $s_i = (1-\rho)$ for every $i \not= 1$. Moreover, $v_1 = \frac{1}{\sqrt{m}}\vec{\mathbf{1}} $ and $v_i$ for $i \not= 1$ is such that the sum of its entries is equal to zero, that is, $v_{i}^\trnsp \vec{\mathbf{1}}  = 0$.
\end{proposition}

Following \citet{hastie_surprises_2019}, let us define:
\begin{equation}
\widehat{H}_{n}(s) =\frac{1}{m} \sum_{i=1}^{m} I{\left\{s \geq s_{i}\right\}},
\end{equation}
where $I$ is the indicator function, and is equal to one when $s \geq s_{i}$ and equal to $0$ otherwise. For the equicorrelated matrix, we have that:
\begin{equation*}
    \widehat{H}_{n}(s) = \frac{m-1}{m}I\left\{s \geq \left(1 - \rho\right)\right\} +
     \frac{1}{m}I\left\{s \geq \left(1 +(m-1) \rho \right)\right\}.
\end{equation*}
Hence, $\widehat{H}_{n}(s) \rightarrow H(s) = I\left\{s \geq 1 - \rho\right\}$ at all continuity points. Hence, $dH= \delta_{1-\rho}$ and we have $\ c_0 = \frac{1}{\gamma(\gamma -1)(1-\rho)}$ which when replaced in Corollary 2 establishes the result.\footnote{There is a small typo in  Corollary 2 from \citet{hastie_surprises_2019}. The term $r^2$ should have appeared multiplying the integral for the overparameterized case.}

We illustrate in Fig.~\ref{fig:equicorrelated-l2} empirical experiments and asymptotic results for equicorrelated feature models under $\ell_2$-adversarial attacks. In Fig.~\ref{fig:equicorrelated-lp} we illustrate it for $\ell_p$ attacks when $p \in \{1.5, 2, 20\}$ and in Fig.~\ref{fig:equicorrelated-linf}, for $p=\infty$.

\begin{figure}
    \centering
    \subfloat[$\eta(m) = 1$]{\includegraphics[height=0.19\columnwidth]{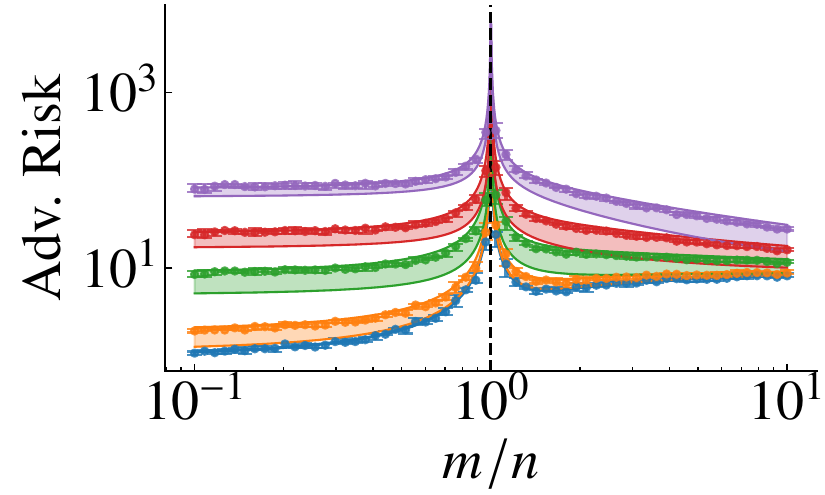}}
    \subfloat[$\eta(m) = \sqrt{\log{m}}$]{\includegraphics[height=0.19\columnwidth]{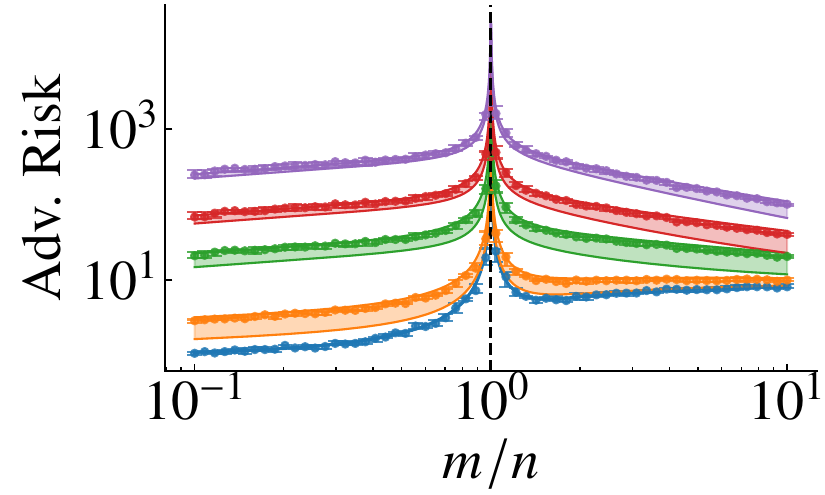}}
    \subfloat[$\eta(m) = \sqrt{m}$]{\includegraphics[height=0.19\columnwidth]{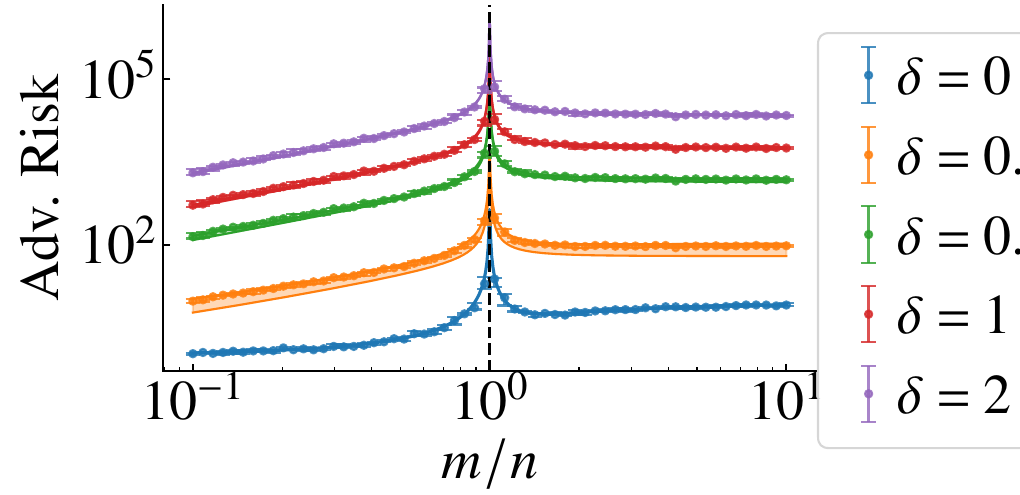}}
    \caption{\emph{Adversarial $\ell_2$-risk, equicorrelated features.} 
     The solid line show the upper and lower bounds on the asymptotic risk obtained from Lemma~\ref{thm:equicorrelated-features}. The results are for equicorrelated features with $r^2 =4$, $\sigma^2 = 1$, $\rho=0.5$.  The error bars give the median and the 0.25 and 0.75 quantiles obtained from numerical experiments (10 realizations) with a fixed training dataset of size $n=300$. We show the results for different scaling choices.}
    \label{fig:equicorrelated-l2}
\end{figure}

\begin{figure}
    \centering
    \subfloat[$\eta(m) = 1$]{\includegraphics[width=0.33\columnwidth]{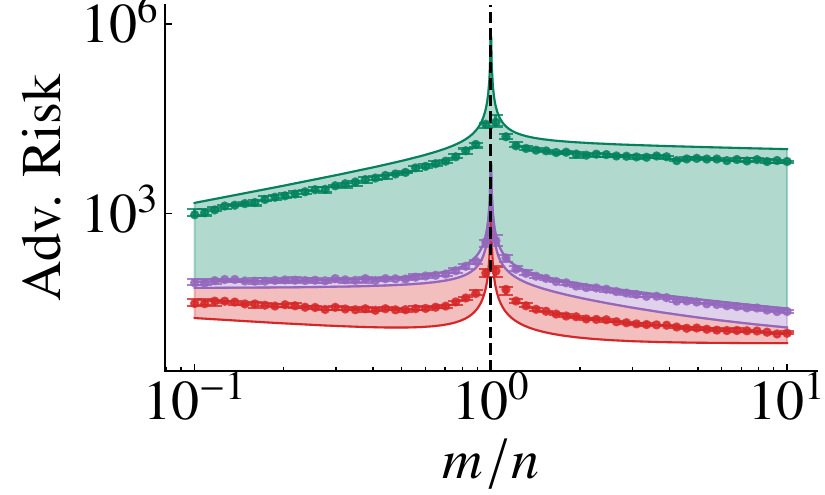}}
    \subfloat[$\eta(m) = \sqrt{\log{m}}$]{\includegraphics[width=0.33\columnwidth]{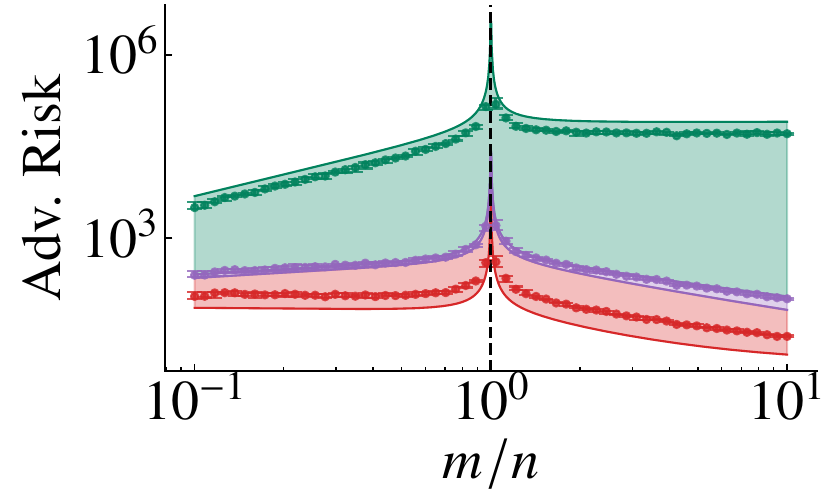}}
    \subfloat[$\eta(m) = \sqrt{m}$]{\includegraphics[width=0.33\columnwidth]{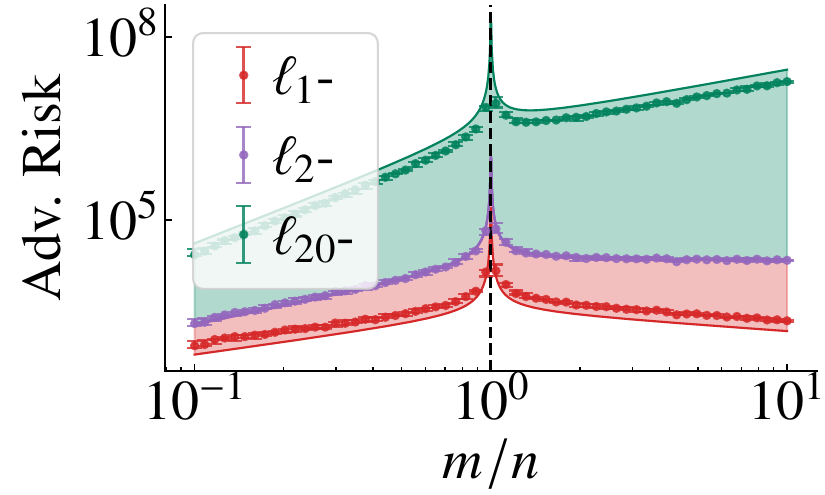}}
    \caption{\emph{Adversarial $\ell_p$-risk, equicorrelated features.} The setup is the same as in Fig.~\ref{fig:equicorrelated-l2}, but for $\ell_p$-adversarial attacks, $p\in \{ 1, 2,\infty\}$ and $\delta=2$. We show the results for different scaling choices.}
    \label{fig:equicorrelated-lp}
\end{figure}

\begin{figure}
    \centering
    \subfloat[$\ell_\infty$-adversarial Risk]{\includegraphics[width=0.5\columnwidth]{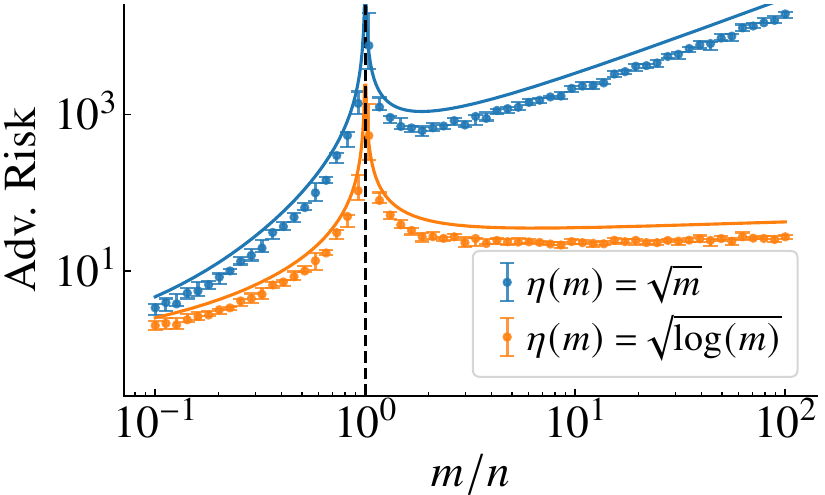}}
    \subfloat[$\|\mhat{\beta}\|_1$]{\includegraphics[width=0.5\columnwidth]{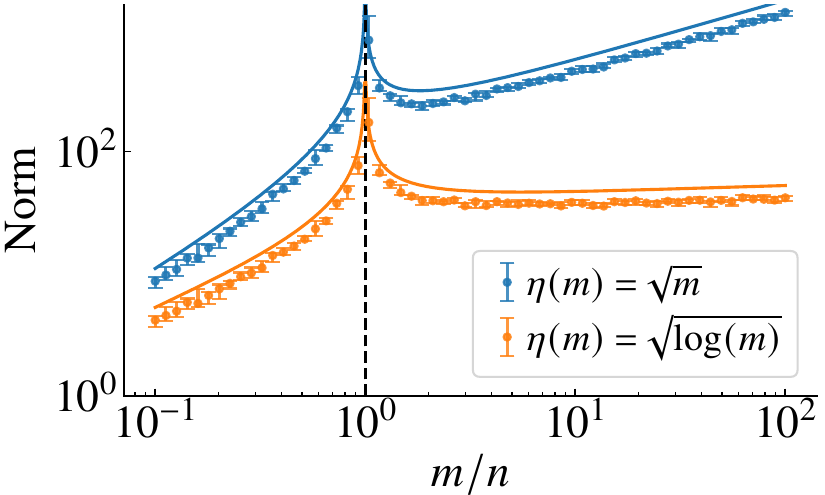}}
    \caption{\emph{Adversarial $\ell_\infty$-risk and $\ell_1$ parameter norm, equicorrelated features.} The asymptotic upper bound is indicated by the full trace.  The error bars give the median and the 0.25 and 0.75 quantiles obtained from numerical experiments (10 realizations). The analysis is performed for a fixed training dataset of size $n=100$ and  for adversarial disturbances of magnitude $\delta = 0.1$. The results are for $r^2 =1$, $\sigma^2 = 1$, $\rho=0.5$. We show the results for two different scaling: $\eta(m) = \sqrt{m}$ and $\eta(m) = \sqrt{\log(m)}$. }
    \label{fig:equicorrelated-linf}
\end{figure}

\section{Latent space model}
\label{sec:latent-model-extension}

Here we show that the latent model described in Equations~\eqref{eq:latent-model-features} and~\eqref{eq:latent-model-outputs} is actually equivalent to a special case of the linear model described in Eq.~\eqref{eq:linear-data-model} for the case where the covariates and the noise are both normal, i.e. $x_i \sim \N(0, \Sigma)$ and $\epsilon_i \sim \N(0, \sigma)$. In both formulations, the pair $(x_i, y_i) \in \R^{m+1}$ is jointly a multivariate Gaussian with zero mean. By matching the covariances, we can conclude that the formulations are equivalent for
\begin{eqnarray}
    \beta &=& W (I_d + W^\trnsp W)^{-1} \theta, \nonumber\\
    \Sigma &=& I_p + W W^\trnsp, \label{eq:equivalent-parameters}\\
    \sigma^2 &=& \sigma^2_{\xi} + \theta^\trnsp (I_d + W^\trnsp W)^{-1}\theta. \nonumber
\end{eqnarray}
That is, from  Eq.~\eqref{eq:linear-data-model} the covariance matrix of the joint Gaussian vector $(x_i, y_i) \in \R^{m+1}$ is given by 
\begin{equation}
    \begin{bmatrix}
        \Exp{x_i x_i^\trnsp} & \Exp{y_i x_i} \\
        \Exp{y_i x_i^\trnsp} & \Exp{y_i y_i}
    \end{bmatrix}
    = 
    \begin{bmatrix}
        \Sigma &  \beta^\trnsp \Sigma  \\
        \Sigma \beta  & \beta^\trnsp \Sigma \beta + \sigma^2
    \end{bmatrix}.
\end{equation}
On the other hand, the latent model~\eqref{eq:latent-model-features} and~\eqref{eq:latent-model-outputs} implies the following covariance matrix for this same vector
\begin{multline*}
    \begin{bmatrix}
        \Exp{W z_i z_i^\trnsp W^\trnsp + u_i u_i^\trnsp} &  \Exp{\theta^\trnsp z_i z_i^\trnsp W + \xi_i u_i} \\
        \Exp{\theta^\trnsp z_i z_i^\trnsp W + \xi_i u_i^\trnsp} & \Exp{\theta^\trnsp z_i z_i^\trnsp \theta + \xi_i \xi_i}
    \end{bmatrix}
    =\\
=\begin{bmatrix}
        W W^\trnsp + I_d &  \theta^\trnsp W^\trnsp\\
        W\theta  & \theta^\trnsp\theta + \sigma_\xi^2
    \end{bmatrix}.
\end{multline*}

Using the matrix inversion lemma $(I_d + W^\trnsp W)^{-1} = I_d + W^\trnsp (I - W W^\trnsp)^{-1}W$ and the identity
$(I_d + W^\trnsp W)^{-1} W^\trnsp = W (I_d + W W^\trnsp)^{-1} $ it is easy to check that~\eqref{eq:equivalent-parameters} renders the two covariance matrices to be equal and the two data generation procedures equivalent.

Now, the hypothesis that $W^\trnsp W = \nicefrac{m}{d} I_d$ makes the calculation of the asymptotics easy. First notice that: 
$\|\beta\|^2_2 = \beta^\trnsp \beta = \theta^\trnsp (I_d + W^\trnsp W)^{-1} W^T W (I_d + W^\trnsp W)^{-1}  \theta$
using the fact that $W^T W = \nicefrac{m}{d} I_d$ we obtain:
$\|\beta\|^2_2  = \frac{\nicefrac{m}{d} }{(1 + \nicefrac{m}{d} )^2} \|\theta\|_2^2$.

Moreover, it is quite straightforward to compute the eigenvalues and eigenvectors of $\Sigma = W W^\trnsp + I_p$. From the hypothesis, it follows that $\Sigma W = (1 + \nicefrac{m}{d})W$. Hence, the first $d$ eigenvectors $v_i$ are the columns of $\sqrt{\nicefrac{d}{m}} W$ and the first $d$ eigenvalues are equal to $s_i = (1 + \nicefrac{m}{d})$. Now, the remaining eigenvectors are orthogonal to the columns of $W$ such that $W^\trnsp v_i = 0$ for $i > d$. The corresponding eigenvalues would be $s_i = 1$. 

Following \citet{hastie_surprises_2019}, let us define:
\begin{equation}
\widehat{H}(s) =\frac{1}{m} \sum_{i=1}^{m} I{\left\{s \geq s_{i}\right\}}, ~~
\widehat{G}(s) =\frac{1}{\|\beta\|_2^2} \sum_{i=1}^{m}(\beta^\trnsp v_i)^2I{\left\{s \geq s_{i}\right\}}, 
\end{equation}
Hence:
\begin{equation}
\label{eq:widehatH}
\widehat{H}(s) =\frac{d}{m} I\left\{ s \ge (1 + m/d)\right\} + \left(1 - \frac{d}{m}\right) I\left\{ s \ge 1\right\}.
\end{equation}
Moreover, $v_i^\trnsp\beta= 0$ if $i > d$, i.e., $\beta$  is orthogonal to the last $m - d$ eigenvectors. Hence, simple manipulation yields that:
\begin{equation}
\label{eq:widehatG}
\widehat{G}(s) =\frac{\nicefrac{d}{m}\|W^\trnsp \beta\|_2^2}{\|\beta\|_2^2} I\left\{ s \ge (1 + m/d)\right\}  = I\left\{ s \ge (1 + m/d)\right\}.
\end{equation}
Let us define $\psi = \frac{d}{m}$. 
\begin{proposition}
    Let $c_0 = c_0(\psi, \gamma)$ be the unique non-negative solution of the following second-order equation:
    \begin{equation}
    1-\frac{1}{\gamma}=\frac{1-\psi}{1+c_{0} \gamma}+\frac{\psi}{1+c_{0}\left(1+\psi^{-1}\right) \gamma}.
    \end{equation}
    Define:\footnote{We define $\mathscr{B}(\psi, \gamma)$ and $\mathscr{V}(\psi, \gamma)$ slightly different from~\citep{hastie_surprises_2019}. The reason is to make explicit the role of $r^2$ and $\sigma^2$. The formulas are equivalent.}
    \begin{equation}
    \begin{aligned}
    \mathscr{B}(\psi, \gamma) &=\left\{1+\gamma c_{0} \frac{\mathscr{E}_{1}(\psi, \gamma)}{\mathscr{E}_{2}(\psi, \gamma)}\right\} \cdot \frac{\left(1+\psi^{-1}\right)}{\left(1+c_{0} \gamma\left(1+\psi^{-1}\right)\right)^{2}}, \\
    \mathscr{V}(\psi, \gamma) &=  \gamma c_{0} \frac{\mathscr{E}_{1}(\psi, \gamma)}{\mathscr{E}_{2}(\psi, \gamma)}, \\
    \mathscr{E}_{1}(\psi, \gamma) &=\frac{1-\psi}{\left(1+c_{0} \gamma\right)^{2}}+\frac{\psi\left(1+\psi^{-1}\right)^{2}}{\left(1+c_{0} \gamma\right)^{2}}, \\
    \mathscr{E}_{2}(\psi, \gamma) &=\frac{1-\psi}{\left(1+c_{0} \gamma\right)^{2}}+\frac{1+\psi}{\left(1+c_{0}\left(1+\psi^{-1}\right) \gamma\right)^{2}},
    \end{aligned}
    \end{equation}
then:
\begin{eqnarray}
    R &\rightarrow& 
    \begin{cases}
    \sigma^2 \frac{\gamma}{1 - \gamma}, \gamma < 1, \\
    r^2 \mathscr{B}(\psi, \gamma)  + \sigma^2  \mathscr{V}(\psi, \gamma), \gamma > 1.
    \end{cases}
    \\
    L_2  &\rightarrow &
    \begin{cases}
    r^2 + \sigma^2 \frac{\gamma}{(1 - \gamma)(1 + \psi)}, \gamma < 1, \\
    r^2 \frac{c_0 \gamma (1 + \psi^{-1})}{1 + c_0 \gamma (1 + \psi^{-1})} + \sigma^2 c_0 \gamma, \gamma > 1.
    \end{cases}
\end{eqnarray}
\end{proposition}

The proposition is (partially) provided by~\citet[Corollary 4]{hastie_surprises_2019} and extended here to also state the asymptotic for $L_2$ and for $R$ in the underparameterized region, both results actually follow from the developments in~\citep{hastie_surprises_2019}.
In the overparameterized region, the proof follows by replacing $\widehat{G}$ and $\widehat{H}$ from Eq.~\eqref{eq:widehatH} and Eq.~\eqref{eq:widehatG} in Definition 1 from~\citet{hastie_surprises_2019} and by further simplifying the expressions. It then follows from the (more general) Theorem 2 in \citep{hastie_surprises_2019} that $R \rightarrow r^2\mathscr{B}(\psi, \gamma) + \sigma^2\mathscr{V}(\psi, \gamma)$.  In the underparameterized region, the result can be obtained using Theorem 1 in \citep{hastie_surprises_2019}.  The convergence of the $L_2$ norm follows from similar analysis using Corollary 3 in \citep{hastie_surprises_2019}.

\begin{figure}[H]
    \centering
    \subfloat[$\ell_2$-adversarial attacks]{\includegraphics[width=0.48\columnwidth]{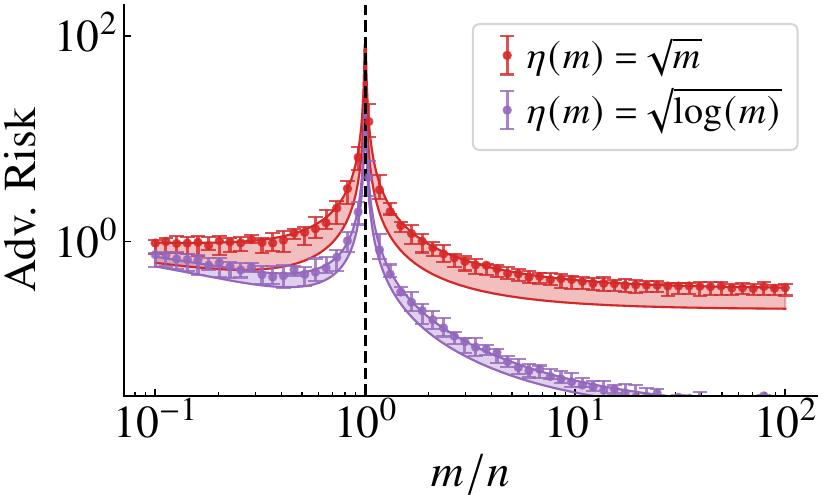}}
    \subfloat[$\ell_\infty$-adversarial attacks]{\includegraphics[width=0.48\columnwidth]{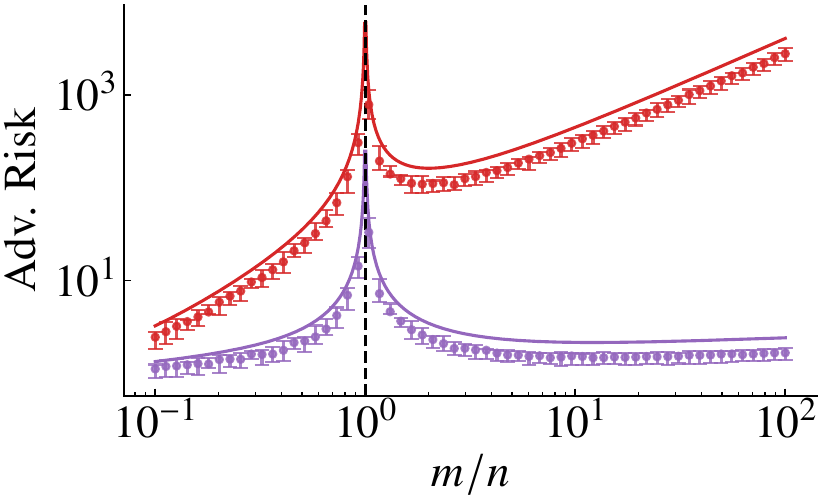}}
    \caption{\emph{Adversarial $\ell_p$-risk, latent space model.} The setup is the same as in Fig.~\ref{fig:latent-l2-and-linf}, but  we plot the asymptotics obtained from the above proposition. In (a), we show the upper and lower bounds; in (b), we show only the upper bound. The argument for why the upper bound is followed closely for $\ell_\infty$-adversarial attacks is provided in the main text.}
    \label{fig:latent-lp-asymptotics}
\end{figure}

\begin{figure}[H]
    \centering
    \subfloat[$p =2, \eta(m) = \sqrt{m}$]{\includegraphics[width=0.48\columnwidth]{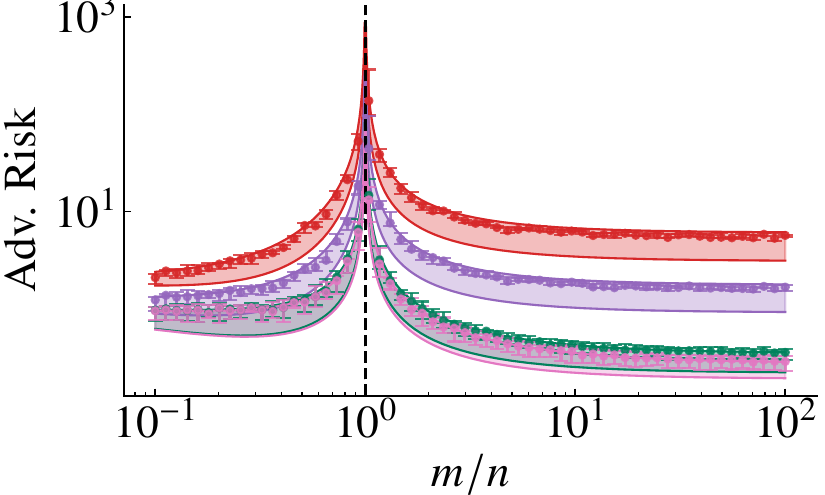}}
    \subfloat[$p =2, \eta(m) = \sqrt{\log(m)}$]{\includegraphics[width=0.48\columnwidth]{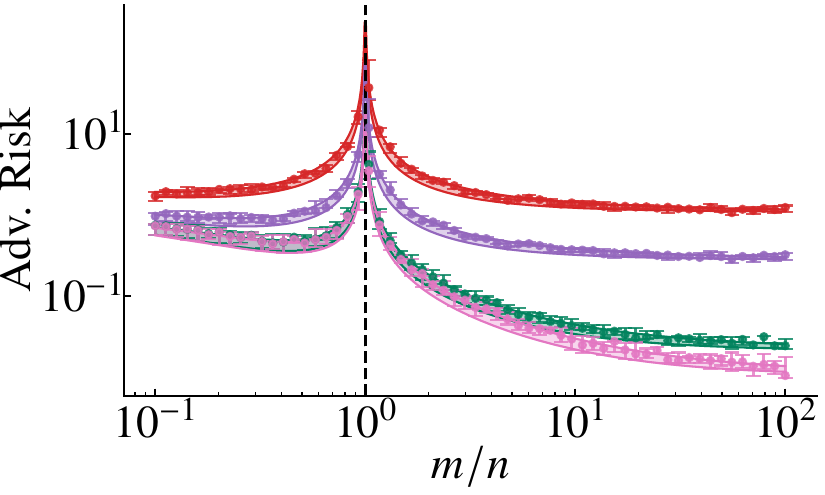}}\\
    \subfloat[$p =\infty, \eta(m) = \sqrt{m}$]{\includegraphics[width=0.48\columnwidth]{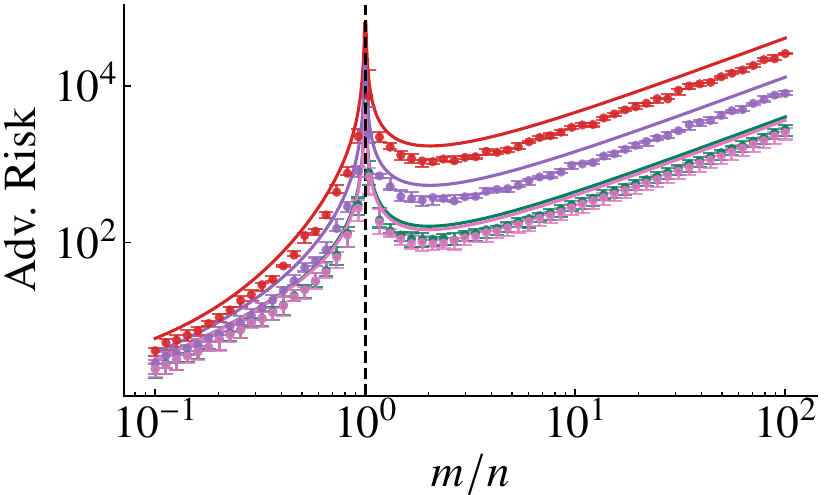}}
    \subfloat[$p =\infty, \eta(m) = \sqrt{\log(m)}$]{\includegraphics[width=0.48\columnwidth]{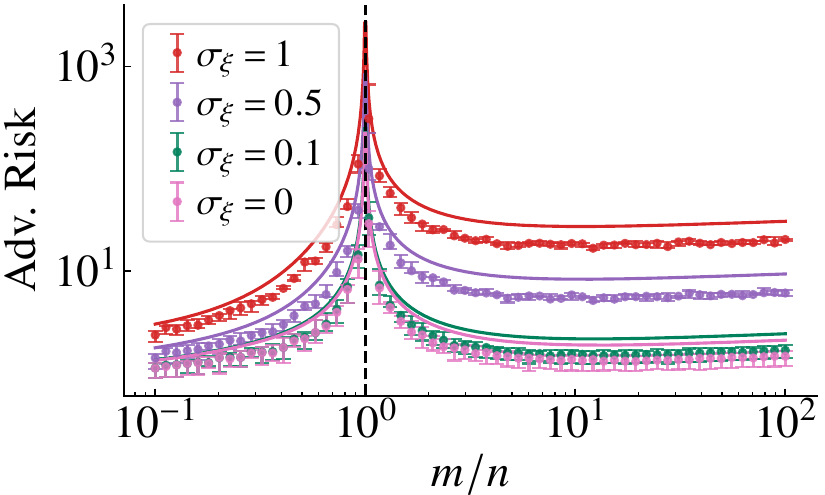}}
    \caption{\emph{Adversarial risk, latent space model, different values of $\sigma_\xi$.} The setup is the same as in Fig.~\ref{fig:latent-lp-asymptotics} but for different values of noise $\sigma_\xi$.}
    \label{fig:latent-lp-noise}
\end{figure}

\begin{figure}[H]
    \centering
    \subfloat[$p =2, \eta(m) = \sqrt{m}$]{\includegraphics[width=0.48\columnwidth]{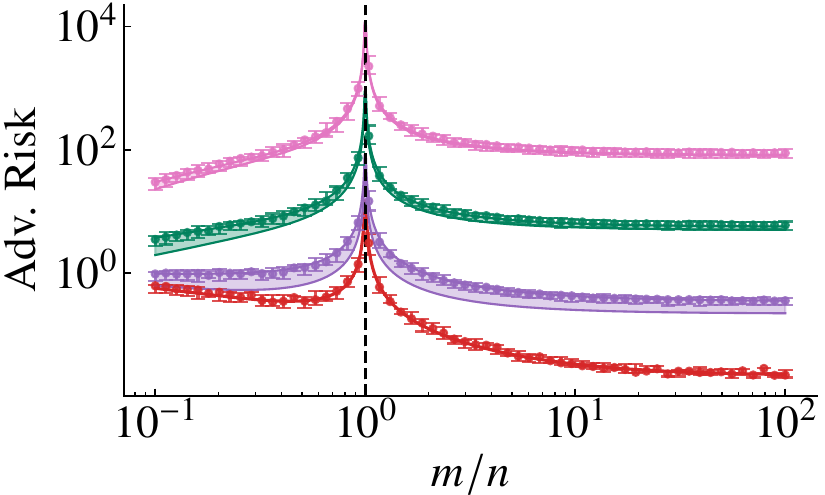}}
    \subfloat[$p =2, \eta(m) = \sqrt{\log(m)}$]{\includegraphics[width=0.48\columnwidth]{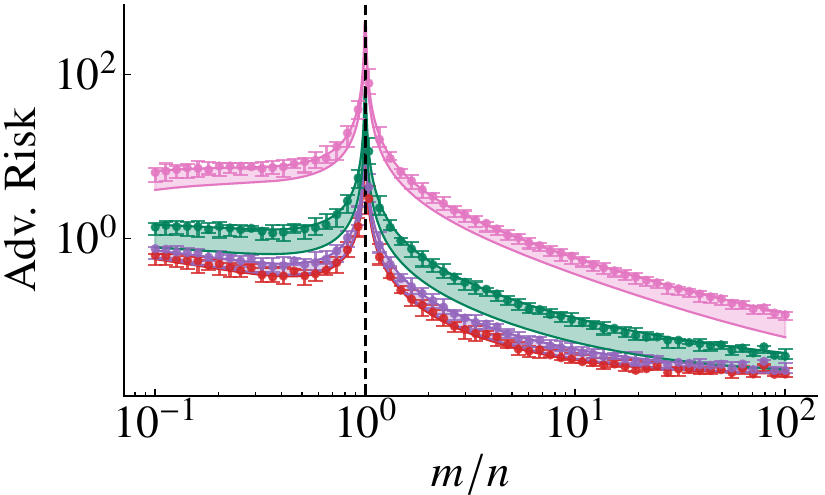}}\\
    \subfloat[$p =\infty, \eta(m) = \sqrt{m}$]{\includegraphics[width=0.48\columnwidth]{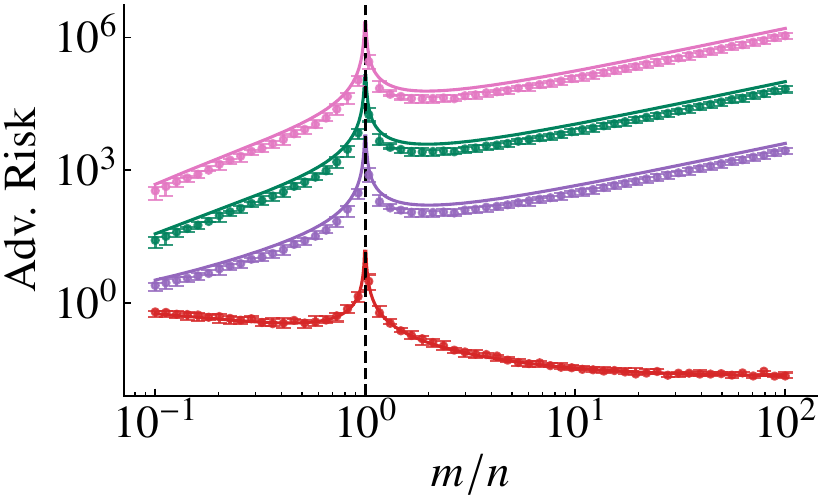}}
    \subfloat[$p =\infty, \eta(m) = \sqrt{\log(m)}$]{\includegraphics[width=0.48\columnwidth]{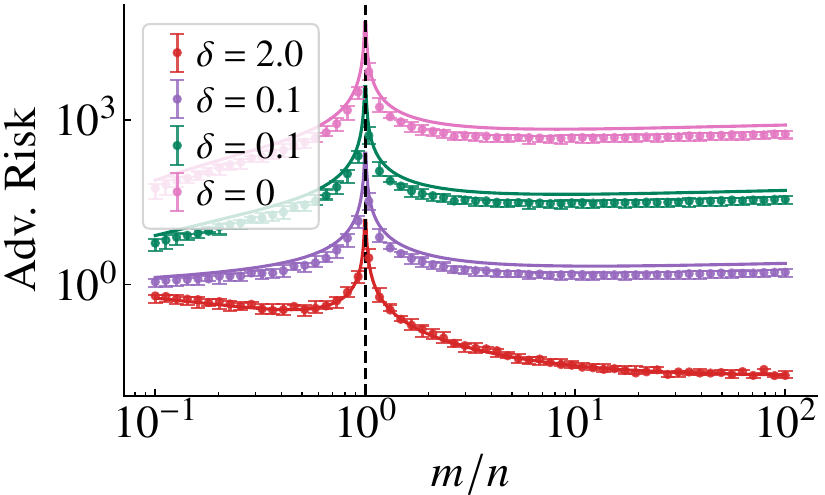}}
    \caption{\emph{Adversarial risk, latent space model, different $\delta$.} The setup is the same as in Fig.~\ref{fig:latent-lp-asymptotics} but for different values of~$\delta$. Here, $\delta=0$ correspond to no adversarial attack.}
    \label{fig:latent-lp-delta}
\end{figure}

\section{Comparison between regularization methods in isotropic features model}
\label{sec:regularization-isotropic-model}

  \begin{figure}[H]
    \centering
    \subfloat[Ridge regression]{\includegraphics[width=0.5\columnwidth]{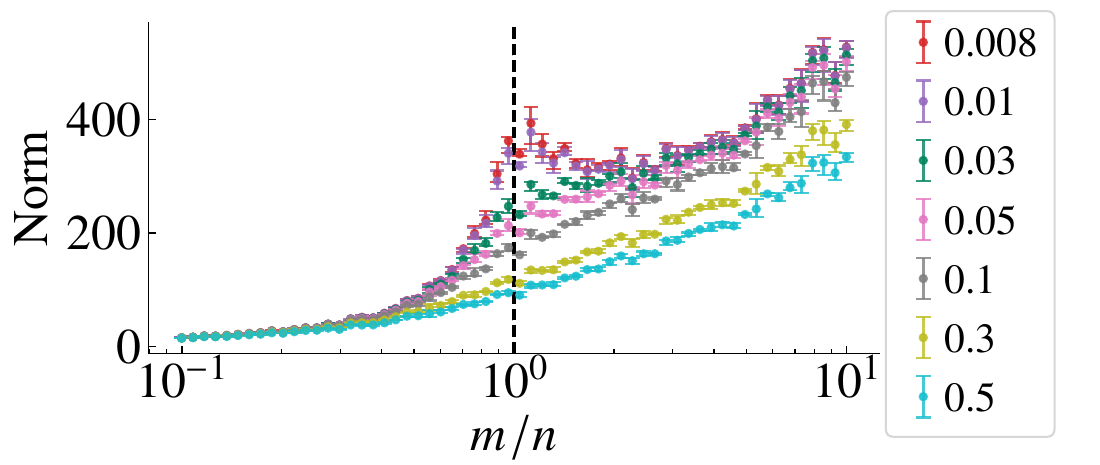}}
    \subfloat[Adversarial training $\ell_2$ ]{\includegraphics[width=0.5\columnwidth]{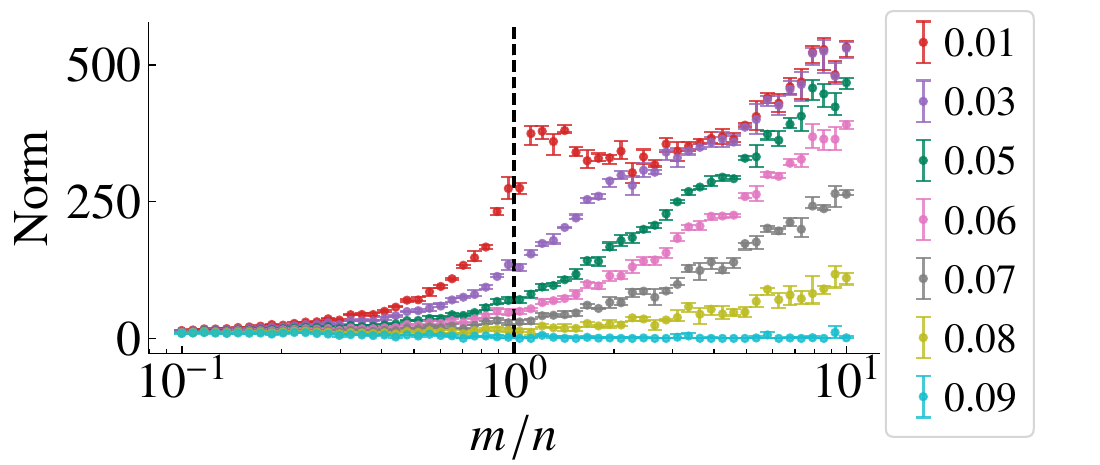}}\\
    \subfloat[Lasso regression]{\includegraphics[width=0.5\columnwidth]{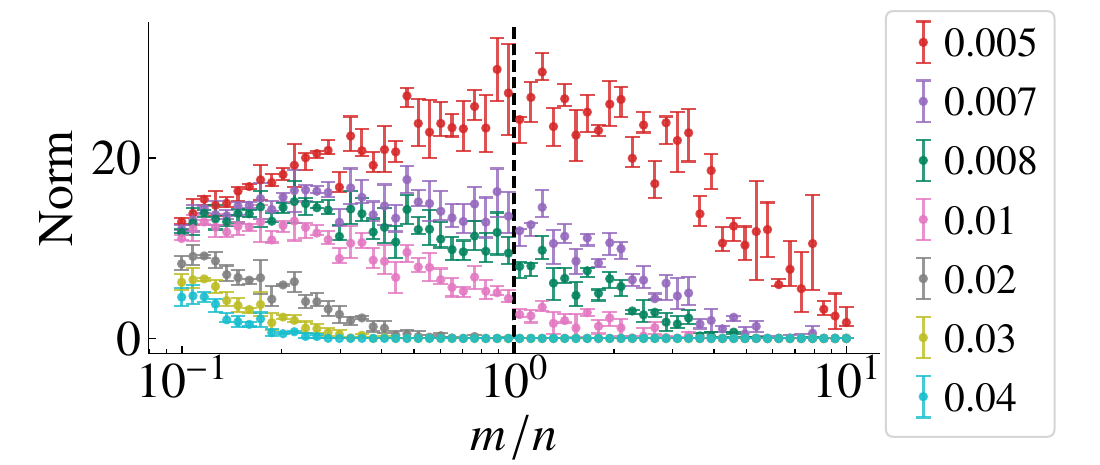}}
    \subfloat[Adversarial training  $\ell_\infty$  ]{\includegraphics[width=0.5\columnwidth]{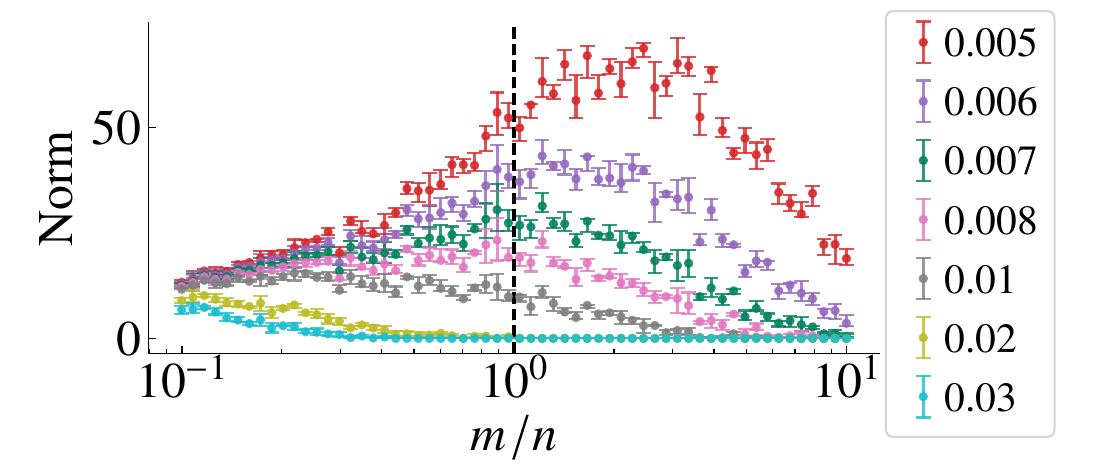}}\\
    \caption{\emph{Parameter norm  $\|\mhat{\beta}\|_1$.} The input variables are scaled with $\eta(m) = \sqrt{m}$. The error bars give the median and the 0.25 and 0.75 quantiles obtained from the numerical experiment (4 realizations) for a fixed training dataset of size $n=100$. The values of $\delta$ for each experiment are indicated in the plot. We repeat the experiment for different amounts of regularization. The regularization parameter $\delta$ is defined in Eq.~\eqref{eq:empirical_advrisk} for the adversarial training and in Eq.~\eqref{eq:ridge} and~\eqref{eq:lasso} for ridge regression and lasso.}
    \label{fig:isotropic-norm-advtraining}
  \end{figure}
  
  \begin{figure}[H]
    \centering
    \subfloat[$\ell_\infty$-adversarial risk]{\includegraphics[width=0.5\columnwidth]{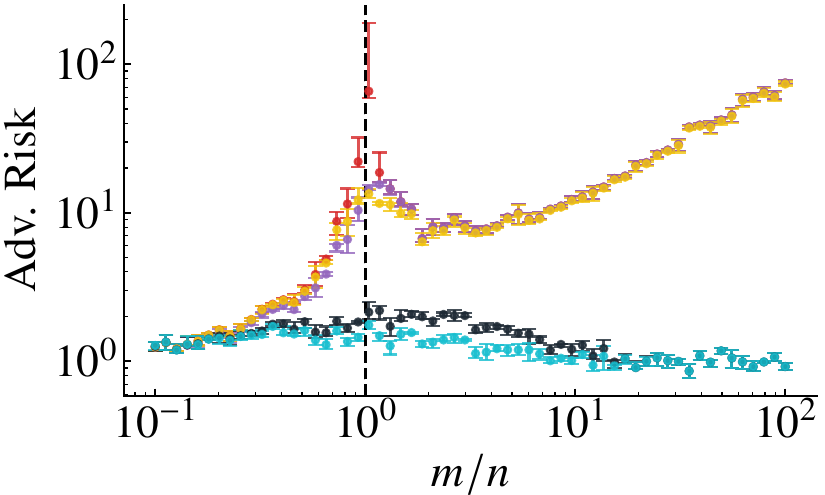}}
    \subfloat[Standard risk ]{\includegraphics[width=0.5\columnwidth]{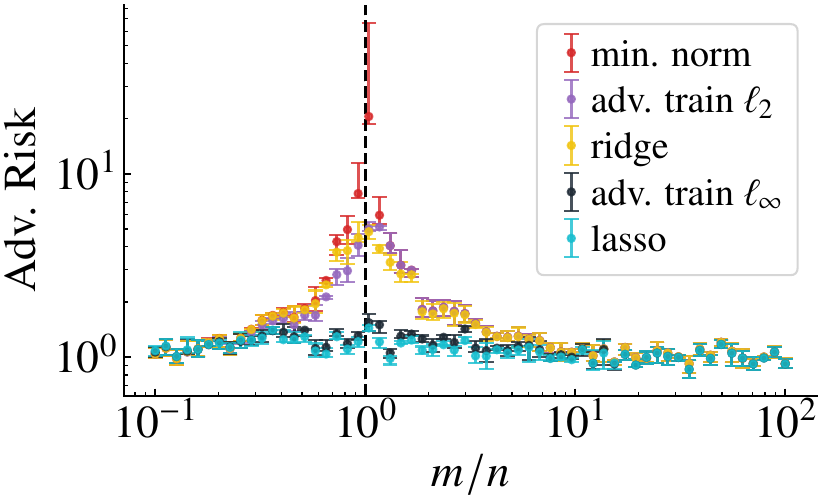}}
 \caption{\emph{Adversarial $\ell_\infty$-risk.} On the $y$-axis we show the risk for models obtained by different training methods. In (a), it is the $\ell_\infty$-adversarial risk; and, in (b) the standard risk. On the $x$-axis we have the ratio between the number of features $m$ and the number of training datapoints $n$. The error bars give the median and the 0.25 and 0.75 quantiles obtained from numerical experiment (4 realizations). We use $\delta=0.01$ during both inference (to compute the adversarial risk) and adversarial training, as in Eq.~\eqref{eq:empirical_advrisk}. We also use $\delta=0.01$ for lasso and ridge regression, see Eq.~\eqref{eq:ridge} and~\eqref{eq:lasso}.}
    \label{fig:isotropic-testerror-advtraining}
  \end{figure}

\end{document}